%% file: jud_iclr.tex
\documentclass{article} %
\usepackage{iclr2027_conference,times}

\input{math_commands.tex}

\usepackage{hyperref}
\usepackage{url}

\usepackage{amsmath}
\renewcommand{\eqref}[1]{\textup{(\ref{#1})}}
\usepackage{amsthm}

\usepackage[ruled,vlined]{algorithm2e}
\usepackage[capitalize,noabbrev]{cleveref}
\usepackage{array}
\usepackage{diagbox}
\usepackage{booktabs}
\usepackage{xcolor}
\usepackage{graphicx}
\usepackage{capt-of}
\definecolor{myhighlight}{HTML}{D55E00}

\graphicspath{{figures/}{Jump_up_down/ICLR2027/iclr2027/figures/}}

\usepackage{siunitx}
\usepackage{amssymb} 
\ExplSyntaxOn

\clist_new:N \g_idlm_lwa_preserve_clist
\clist_gset:Nn \g_idlm_lwa_preserve_clist
  {
    va,vb,vc,vd,ve,vf,vg,vh,vi,vj,vk,vl,vm,
    vn,vo,vp,vq,vr,vs,vt,vu,vw,vx,vy,vz,rv
  }

\clist_map_inline:Nn \g_idlm_lwa_preserve_clist
  {
    \cs_gset_eq:cc { g_idlm_lwa_saved_#1 } {#1}
    \cs_undefine:c {#1}
  }

\AddToHook{package/letterswitharrows/after}
  {
    \clist_map_inline:Nn \g_idlm_lwa_preserve_clist
      {
        \cs_gset_eq:cc {#1} { g_idlm_lwa_saved_#1 }
      }

    \cs_undefine:N \ru
    \cs_undefine:N \rd
  }

\ExplSyntaxOff

\usepackage[presets={abc,vec-cev}]{letterswitharrows}\crefname{equation}{Equation}{Equations}
\usepackage{pifont}

\usepackage{listings}

\definecolor{algteal}{HTML}{00897B}
\definecolor{algorange}{HTML}{F57C00}
\definecolor{algmagenta}{HTML}{D81B60}
\definecolor{algviolet}{HTML}{6D28D9}
\definecolor{algblue}{HTML}{2563EB}

\newcommand{\algfn}[1]{\textnormal{\texttt{\textcolor{algmagenta}{\detokenize{#1}}}}}
\newcommand{\algid}[1]{\textnormal{\texttt{\textcolor{algviolet}{\detokenize{#1}}}}}

\newcommand{\algcom}[1]{\textnormal{{\ttfamily\textcolor{algteal}{\# #1}}}}

\newcommand{\idlmalgcaption}[3]{%
  \refstepcounter{algocf}\label{#1}%
  \noindent\rule{\linewidth}{0.8pt}\par\vspace{0.5mm}%
  \noindent{\bfseries #2}\par\vspace{0.3mm}%
  {\scriptsize #3}\par\vspace{0.5mm}%
  \noindent\rule{\linewidth}{0.4pt}\par\vspace{-1mm}%
}

\newcommand{\idlmalgbottom}{\vspace{-1.5mm}\noindent\rule{\linewidth}{0.4pt}}
\lstdefinelanguage{IDLMPseudocode}{
  basicstyle=\ttfamily\small,
  keywordstyle=\color{black},
  commentstyle=\color{black},
  stringstyle=\color{black},
  morekeywords={},
  showstringspaces=false,
  columns=fullflexible,
  keepspaces=true,
  frame=none
}

\usepackage{amsmath,amssymb}
\usepackage[most]{tcolorbox}

\newtcolorbox{goalbox}[1][Goal]{%
  enhanced,
  breakable,
  colback=white,
  colframe=black!40,
  boxrule=0.8pt,
  arc=1.6mm,
  left=3mm,
  right=3mm,
  top=2.5mm,
  bottom=2.5mm,
  before skip=10pt,
  after skip=10pt,
  attach boxed title to top left={%
    xshift=5mm,
    yshift*=-\tcboxedtitleheight/2
  },
  boxed title style={%
    colback=white,
    colframe=white,
    boxrule=0pt,
    sharp corners,
    left=1mm,
    right=1mm,
    top=0pt,
    bottom=0pt
  },
  title={#1},
  fonttitle=\bfseries,
  coltitle=black
}
\usepackage{aliascnt}
\newcommand{\newsharedtheorem}[3]{%
  \newaliascnt{#1}{theorem}%
  \newtheorem{#1}[#1]{#2}%
  \aliascntresetthe{#1}%
  \crefname{#1}{#2}{#3}%
}
\theoremstyle{plain}
\newtheorem{theorem}{Theorem}[section]
\newsharedtheorem{proposition}{Proposition}{Propositions}
\newsharedtheorem{lemma}{Lemma}{Lemmas}
\newsharedtheorem{corollary}{Corollary}{Corollaries}
\newsharedtheorem{example}{Example}{Examples}
\newsharedtheorem{claim}{Claim}{Claims}
\theoremstyle{definition}
\newsharedtheorem{definition}{Definition}{Definitions}
\newsharedtheorem{assumption}{Assumption}{Assumptions}
\theoremstyle{remark}
\newsharedtheorem{remark}{Remark}{Remarks}
\newcommand{\Def}{\textbf}

\newcommand{\N}{\mathbb N}
\newcommand{\Z}{\mathbb Z}

\newcommand{\x}{\mathbf{x}}

\newcommand{\y}{\mathbf{y}}
\newcommand{\n}{\mathbf{n}}
\newcommand{\z}{\mathbf{z}}

\newcommand{\kk}{\mathbf{k}}

\newcommand{\Id}{\operatorname{Id}}
\newcommand*\diff{\mathop{}\!\mathrm{d}}

\newcommand{\w}{w}

\newcommand{\TT}{\tau}
\newcommand{\dd}{D}

\newcommand{\idx}{i}

\newcommand{\target}{p_{\textnormal{data}}}
\newcommand{\pos}{\operatorname{Poisson}}
\newcommand{\bin}{\operatorname{Binomial}}
\newcommand{\ske}{\operatorname{Skellam}}
\newcommand{\poss}{\Pi}
\newcommand{\binn}{\mathcal B}
\newcommand{\trans}{p}
\newcommand{\transp}{q}

\newcommand{\PP}{\mathbb P}
\newcommand{\EE}{\mathbb E}

\newcommand{\data}{\x_{\textnormal{data}}}
\newcommand{\possrv}{\mathrm N}
\newcommand{\binnrv}{\mathrm B}
\newcommand{\urv}{\mathrm U}
\newcommand{\drv}{\mathrm D}
\newcommand{\lrv}{\mathrm L}

\newcommand{\pp}{X}
\newcommand{\ppp}{Y}
\newcommand{\ppa}{\mathbf{X}}
\newcommand{\pppa}{\mathbf{Y}}

\newcommand{\rate}{Q}

\newcommand{\Ru}
{\mathrm u}

\newcommand{\Rd}
{l}
\newcommand{\Rdi}
{\mathrm d}

\newcommand{\ru}{\dot{\mathrm u}}

\newcommand{\rd}{\dot{l}}
\newcommand{\rdi}{\dot{\mathrm d}}

\newcommand{\den}{\mathrm m_{\textnormal{\texttt{latent}}}}
\newcommand{\dend}{\mathrm m_{\textnormal{\texttt{data}}}}
\newcommand{\denpos}{\mathrm m_{\textnormal{\texttt{latent-up}}}}
\newcommand{\denbin}{\mathrm m_{\textnormal{\texttt{latent-down}}}}

\newcommand{\cin}{c_{\textnormal{in}}}
\newcommand{\cskip}{c_{\textnormal{skip}}}
\newcommand{\cbias}{c_{\textnormal{bias}}}
\newcommand{\ssin}{s_{\textnormal{in}}}
\newcommand{\cout}{c_{\textnormal{out}}}
\newcommand{\net}{F}
\newcommand{\mdata}{\mu}
\newcommand{\vdata}{v}
\newcommand{\cdata}{\textnormal{cov}_{A,Y}}

\newcommand{\var}{\operatorname{Var}}
\newcommand{\cov}{\operatorname{Cov}}

\newcommand{\jud}{JUD}

\title{Jumping up and down: Denoiser diffusion models for discrete ordinal data}

\iclrfinalcopy
\author{Yair Shenfeld \\
Division of Applied Mathematics\\
Brown University\\
Providence, RI, USA 02912\\
\texttt{yair\_shenfeld@brown.edu} \\
\And
Ricardo Baptista \\
Department of Statistical Sciences \\
University of Toronto\\
Toronto, ON, Canada M5G 1X6\\
\texttt{r.baptista@utoronto.ca}\\
\AND
Stefano Peluchetti \\
Sakana AI\\
Tokyo, Japan \\
\texttt{stepelu@sakana.ai}
}

\begin{document}

\maketitle

\begin{abstract}
Diffusion models are highly developed in continuous spaces for image and video domains. Recently, major advances have been made for discrete diffusion models for categorical data, specifically in the language domain. In contrast, diffusion models for discrete integer-valued data are less developed, despite the prevalence of this modality, ranging from images and music to gene counts. We introduce  Jumping Up and Down (\jud)---a new family of denoiser-based diffusion models for discrete ordinal data. This is the first family of diffusion models for ordinal data which centers around training denoisers, which at the same time allows for bi-directional (up and down) perturbations of the data. The simplicity of the training objective, combined with the flexibility of bi-directional perturbations, leads us to obtain competitive results across different data modalities. 
\end{abstract}

\section{Introduction}

\paragraph{Diffusion models for discrete ordinal data.} The success of diffusion models in continuous domains such as image and video is driving an explosion of activity in applying diffusion models for discrete data. The vast majority of diffusion models for discrete data focus on language where the data is \emph{categorical}, that is, a data modality with no natural order. In contrast, in many applications the discrete data also contains a natural order, i.e., the data is  \emph{ordinal} as its distribution is over (subsets of) $\N^{\dd}$ or $\Z^{\dd}$. This modality naturally shows up in images (8-bit digital grayscale from 0 to 255), music (MIDI note numbers from 0 to 127), tabular data (amount in dollars of debt/credit), and in scientific applications ranging from biology (gene expression counts in cells) to chemistry (number of atoms in molecules). In addition to the aforementioned applications of discrete ordinal data, this modality  is also interesting from the theoretical perspective as it lies in some sense between categorical and continuous data: On the one hand, the data is discrete as in categorical data, while on the other hand the data has an intrinsic order as in continuous data in $\R^{\dd}$. Despite the importance of discrete ordinal data, the theory and practice of diffusion models for this modality are still developing. In this paper we contribute to both.

\paragraph{Denoisers.} The essence of diffusion models is learning clean data from progressively noisy versions of the same data point. In continuous space, the noise is often taken to be Gaussian with the level of noise corresponding to the variance of the Gaussian. The choice of Gaussian noise is very convenient since it provides a clean relation between \emph{denoising} and \emph{sampling}. First, one learns the \Def{data denoiser},
\begin{equation}
\label{eq:denoiser_intro_intro}
\mathrm{m}_{\textnormal{data}}(\x,t):=\EE[\data\,|\, \texttt{CorruptData}( \data,t)=\x],\qquad\qquad \data\sim\target,
\end{equation}
where $\texttt{CorruptData}( \data,t)$ can be, for example, adding to $\data$ Gaussian noise with variance $t$. Second, the data denoiser $\mathrm{m}_{\textnormal{data}}$ can be directly used for sampling via \Def{Tweedie's formula}, which connects the denoiser $\mathrm{m}_{\textnormal{data}}$ to the \emph{score} $\nabla\log \trans_t$, where $\trans_t$ is the distribution of $t$-level corrupted data. The validity of Tweedie's formula is due to the use of \emph{Gaussian} noise when corrupting the data. The advantages of formulating a training objective based on the denoiser were recently emphasized in \citet{li2026back, hu2026elf}. In contrast to the continuous setting, the absence of the Gaussian distribution in the discrete setting necessitates the development of a different concept of ``data corruption''. For categorical data, the two prevailing choices for noise are uniform (choosing a token uniformly at random) or masking (adding an extra graveyard token), and the level of noise corresponds to the rate at which a token is flipped to a random or masked token. The lack of ordinal structure in categorical data has led to deviation from training a denoiser as in \cref{eq:denoiser_intro_intro}, since the notion of average (expectation) does not make sense in such generality\footnote{The term ``denoiser'' is sometimes used in the literature on language diffusion models to refer to the prediction of the log probabilities over tokens. Here we use the term ``denoiser'' strictly for the prediction of the data point itself, as in \cref{eq:denoiser_intro_intro}. Below we explain the advantages of directly training the latter.}. Instead, the common training objectives in diffusion models for categorical data are either \emph{discrete scores} \citep{meng2022concrete, lou2024discrete}, which are the analogues of $\nabla\log \trans_t$ in the discrete setting,  or the \emph{rates} of the generative process \citep{DBLP:journals/corr/abs-2205-14987, holderrieth2025generator}. But for ordinal data, the definition of a denoiser analogous to \eqref{eq:denoiser_intro_intro} does make sense, and learning the denoiser has a number of advantages: 
\begin{itemize}
    \item The denoiser  can be trained with a large family of losses (Bregman divergences), including the  simple mean-square error loss. Alternatively, the denoiser can be learned by training cross-entropy objectives for the \emph{marginal} posteriors; see \cref{sec:up_down}.
    \item Techniques from continuous diffusion models, which are denoiser-based, can in principle be adapted to the discrete ordinal setting: EDM scaling \citep{karras2022elucidating}, Distillation
via Moment Matching \citep{salimans2024multistep}, dual score matching \citep{guth2026learning}, blind denoisers \citep{sun2025is,wang2025equilibrium,kadkhodaie2026blind}, etc. See \cref{sec:exp} for our use of the EDM scaling.
\end{itemize}
In Figure~\ref{fig:toy_paths} we compare various noising processes and methods for discrete diffusion models that are trained using the same set of $20,000$ data samples from a bi-modal Negative Binomial Mixture distribution; see \cref{app:synthetic} for more details. As we will demonstrate in this work, these improvements show up not only on toy examples but on real-world datasets as well.

\begin{figure}[!htbp]
\centering
\begin{tabular}{cc}
  \hspace{0.5cm} \textbf{Binomial--Poisson (ours)} & \hspace{0.7cm} \textbf{Count Bridge}{ \footnotesize \citep{fishman2026count}}\\
  \includegraphics[width=0.45\textwidth]{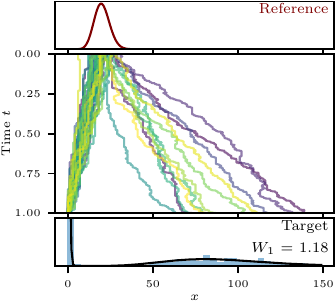} &
  \includegraphics[width=0.45\textwidth]{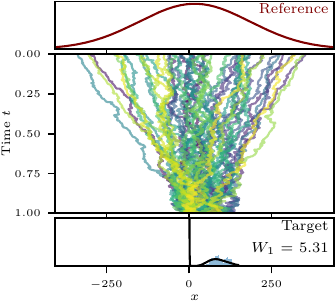} \\
  \hspace{1cm} \textbf{Binomial} { \footnotesize \citep{shenfeld2026binomial}} & \hspace{0.9cm}\textbf{Count Flow} { \footnotesize \citep{wei2026flow}} \\
  \includegraphics[width=0.45\textwidth]{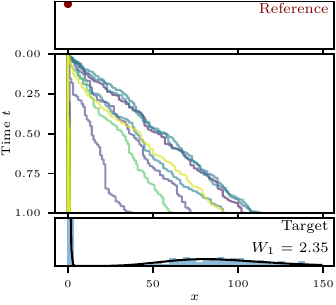} &   \includegraphics[width=0.45\textwidth]{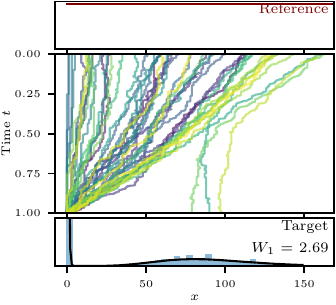} \\[2pt]
  \hspace{0.5cm}\textbf{Poisson} & \hspace{1cm}\textbf{SEDD} {\footnotesize{\citep{lou2024discrete}}} \\
  \includegraphics[width=0.45\textwidth]{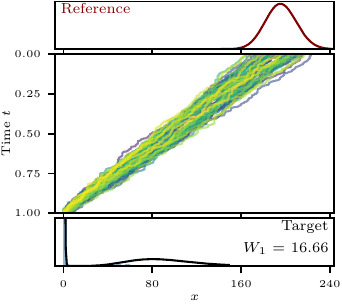} & 
  \includegraphics[width=0.45\textwidth]{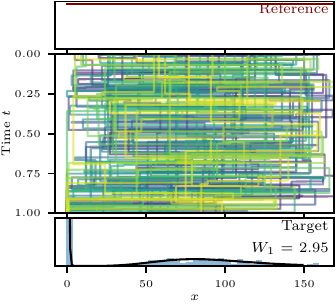}
\end{tabular}
\caption{Sample paths for several noising processes in discrete diffusion models and the resulting histogram of generated samples (blue) relative to the true data distribution (black) \emph{Negative Binomial Mixture}. We observe closest agreement in $W_1$ distance to the data with the \Def{Binomial-Poisson} process proposed in this work.  See \cref{app:synthetic} for details. ~\label{fig:toy_paths}}
\end{figure}

\paragraph{Jumping up and down} In continuous spaces, the Gaussian noise blurs the data in all directions. For ordinal data, a data point in $\N$ or $\Z$ can be perturbed either up, or down, or both. As we will see below, uni-directional perturbations (either just up or just down) give rise to denoiser-based training and sampling (cf. \cref{app:only_up_down}). That is, provided a suitable choice of up-perturbation, or down-perturbation, we can train an appropriate denoiser $\mathrm{m}_{\textnormal{data}}$ and then use it to sample from the data distribution. But perturbing the data in only one direction can lead to inaccurate samplers. This is because once the sampler over/undershoots, due to inevitable approximation errors, the process cannot ``undo'' prior mistakes and correct the trajectory evolution. This is evident in Figure~\ref{fig:toy_paths} when comparing Binomial-Poisson (one of our up-and-down methods), Binomial (only down) and Poisson (only up). Hence, our goal is to find perturbations and samplers which allow for up \emph{and} down jumps, while at the same time being \emph{denoiser-based}. This is a non-trivial task, as the up and down perturbations can interact in arbitrarily complex ways, which was not accomplished in previous works; for space reasons, related works are discussed in \cref{sec:related}. This is precisely the goal that we will achieve in this paper.

\paragraph{Our contributions}

\begin{itemize}
\item We develop  a novel framework of  \Def{denoiser-based} diffusion models (\jud) for discrete ordinal data on $\N^{\dd}$ and $\Z^{\dd}$. Our key innovation is to train \Def{latent} denoisers which predict \Def{partially noised data}. Our main result (\cref{thm:main}) shows that the latent denoisers are sufficient to sample from the data distribution via  a stochastic process which allows for \Def{Up-and-Down} moves.

\item We demonstrate the effectiveness of this framework on synthetic data and real-world data in domains such as images (CIFAR-10), and reasoning (Sudoku completion). In particular, we obtain state-of-the-art FID on CIFAR-10 among discrete diffusion models.  
\end{itemize}

\section{Jumping up and down (\jud)}
\label{sec:up_down}

Continuous diffusion models rely on \emph{Tweedie's formula}. For $\data\sim\target$ on $\R^{\dd}$, add independent Gaussian noise with a continuous schedule $\sigma_t>0$,
\begin{equation}
\label{eq:gauss_perturbation}
Z_t:=\data + G_t,\qquad G_t\sim\mathcal N(0,v_t\Id_{\dd}),\qquad v_t:=\int_0^t\sigma_u^2\,\mathrm du,
\end{equation}
and let $\trans_t$ be the law of $Z_t$, so $\trans_0=\target$. For $t>0$, Tweedie's formula relates the score of $\trans_t$ to the \emph{data denoiser} $\mathrm{m}_{\textnormal{data}}(\x,t):=\EE[\data|Z_t=\x]$:
\begin{equation}
\label{eq:Tweedi_cts}
v_t\nabla\log\trans_t(\x)=\mathrm{m}_{\textnormal{data}}(\x,t)-\x.
\end{equation}
The denoiser can be trained by regression (or any Bregman divergence), providing the score for sampling. For a final time $\TT$, the stochastic differential equation (SDE) in reverse time $s=\TT-t$,
\begin{equation}
\label{eq:sde}
\mathrm d\mathbf{Z}_s=\sigma_{\TT-s}^2\nabla\log \trans_{\TT-s}(\mathbf{Z}_s)\mathrm d s+\sigma_{\TT-s}\mathrm d \mathbf{W}_s,
\end{equation}
with $\mathbf{Z}_0\sim\trans_{\TT}$ and standard $\dd$-dimensional Brownian motion $(\mathbf{W}_s)$, has marginals $\mathbf{Z}_s\sim\trans_{\TT-s}$ and hence $\mathbf{Z}_{\TT}\sim\target$ \citep{song2021scorebased}.

Consider now the setting where $\target$ is a distribution over $\N^{\dd}$ or $\Z^{\dd}$. A Gaussian perturbation will lead to non-integer values so a different type of noise must be considered. In this paper we focus on the generative modeling framework of continuous-time Markov-chains (CTMCs) \citep{DBLP:journals/corr/abs-2205-14987} for discrete data.

In order to have an analogue of \cref{eq:Tweedi_cts} in our setting, we will consider either \emph{Poisson} noise or \emph{Binomial} noise. In detail, 
when $\target$ is supported over $\N^{\dd}$ we perturb a data point $\data\sim\target$ by
\begin{equation}
\label{eq:perturb_intro}  
\pp_t:=\data-\lrv_t+\possrv_t,\qquad\textnormal{with} \qquad \lrv_t\sim\bin(\data,\Rdi_t),\quad \possrv_t\sim \pos(\Ru_t),
\end{equation} 
where $\possrv_t$ is independent of $(\data,\lrv_t)$. Here, $\bin(\data,\Rdi_t)$ is a $\dd$-dimensional product Binomial distribution whose $i$th coordinate has $\data^i$ trials and probability of success $\Rdi^i_t\in [0,1]$, and $\pos(\Ru_t)$ is the $\dd$-dimensional product Poisson distribution with a vector of parameters  $\Ru_t\in \R^{\dd}_{\ge 0}$. Taking $\TT$ to be the final time, and choosing $\Rdi_0=0, \Rdi_{\TT}=1$, and $\Ru_0=0$, we get that the distribution $\trans_t$ of $\pp_t$ interpolates between $\trans_0=\target$ and $\trans_{\TT}=\pos(\Ru_{\TT})$. Note that the choice of noise in \cref{eq:perturb_intro} guarantees that $\pp_t\in \N^{\dd}$ for all $t\in [0,\TT]$. 
  
When $\target$ is supported over $\Z^{\dd}$ we perturb a data point $\data\sim\target$ by
\begin{equation}
\label{eq:perturb_pp_intro}  
\pp_t:=\data-\lrv_t+\possrv_t,\qquad\textnormal{with} \qquad \lrv_t\sim \pos(\Rdi_t),\quad \possrv_t\sim \pos(\Ru_t),
\end{equation} 
where $(\data,\lrv_t,\possrv_t)$ are independent. Taking $\TT$ to be the final time, and choosing $\Rdi_0=\Ru_0=0$, and $\Rdi_{\TT}, \Ru_{\TT}$ large enough, we get that the distribution $\trans_t$ of $\pp_t$ interpolates between $\trans_0=\target$ and $\trans_{\TT}\overset{\textnormal{approx.}}{\sim} \ske(\Rdi_{\TT},\Ru_{\TT})$ (distribution of the difference of independent Poissons; \citealp{skellam1946frequency}). Note that the choice of noise in \cref{eq:perturb_pp_intro} guarantees that $\pp_t\in \Z^{\dd}$ for all $t\in [0,\TT]$.

The perturbations \eqref{eq:perturb_intro}  and \eqref{eq:perturb_pp_intro}  are very simple, but finding CTMCs whose marginals match these perturbations, which at the same time can take advantage of denoiser-based training, is not obvious and has been missing so far. In this work we construct CTMCs whose marginals match \cref{eq:perturb_intro} or \cref{eq:perturb_pp_intro}, which can be trained using only \textbf{latent} denoisers.
Our key innovation is to train, in each of the perturbations \eqref{eq:perturb_intro} and \eqref{eq:perturb_pp_intro}, \Def{latent denoisers} which denoise not the data itself, but \Def{partially noised data}. We prove that training such denoisers suffices to sample from the data distribution. Further, we will show that the data denoiser can be expressed as a \Def{non-linear} and \Def{non-local} function of the latent denoisers. The subsequent sections show how to achieve this goal for target distributions supported on $\N^{\dd}$ or $\Z^{\dd}$. 

\subsection{The Binomial-Poisson process and the Poisson-Poisson process}
In this section we provide our main result which shows how to construct a denoiser-based CTMC whose marginals match the perturbations in \cref{eq:perturb_intro} and \cref{eq:perturb_pp_intro}. We call the CTMC whose marginals match \cref{eq:perturb_intro} the \Def{Binomial-Poisson process}, and the CTMC whose marginals match \cref{eq:perturb_pp_intro} the \Def{Poisson-Poisson process}. 

In the following we let $\{e_i\}_{i\in [\dd]}$ be the one-hot vectors, and given continuously differentiable, coordinatewise strictly increasing noise schedules $\Rdi_t$ (which are in $[0,1]^{\dd}$ for \eqref{eq:perturb_intro} and in $\R^{\dd}_{\ge 0}$ for \eqref{eq:perturb_pp_intro}), and $\Ru_t\in \R^{\dd}_{\ge 0}$, we write  $\rdi_t:=\frac{\diff}{\diff t}\Rdi_t$ and $\ru_t:= \frac{\diff}{\diff t}\Ru_t$.

\begin{theorem}
\label{thm:main}
Let $(\trans_t)_{t\in [0,\TT]}$ be the marginals of $(\pp_t)_{t\in [0,\TT]}$,
\begin{equation}
\label{eq:perturb_intro_thm_main}
\pp_t:=\data-\lrv_t+\possrv_t,
\end{equation}
where  
\begin{center}
\small
\begin{tabular}{c|cccccc}
\toprule
& $\lrv_t$ & $\possrv_t$ & \textnormal{Independence} & $\Rdi_0^i$ & $\Ru_0^i$ & $\Rdi_{\TT}^i$ \\
\midrule
\textnormal{\textbf{Binomial--Poisson}}
& $\bin(\data,\Rdi_t)$
& $\pos(\Ru_t)$
& $(\data,\lrv_t)\perp \possrv_t$
& $0$
& $0$
& $1$ \\
\textnormal{\textbf{Poisson--Poisson}}
& $\pos(\Rdi_t)$
& $\pos(\Ru_t)$
& $\data,\lrv_t,\possrv_t\ \textnormal{independent}$
& $0$
& $0$
& $-$ \\
\bottomrule
\end{tabular}
\end{center}
and define the latent denoisers 
\begin{equation}
\label{eq:denoisers_bin_poss_thm_main}
\denbin(\x,t):=\EE[\lrv_t|\pp_t=\x],\qquad \denpos(\x,t):=\EE[\possrv_t|\pp_t=\x].
\end{equation}
Write $s=\TT-t$ for reverse time. There exists a forward CTMC with the same marginals as $(\pp_t)_{t\in [0,\TT]}$ whose time reversal $(\ppa_s)_{s\in [0,\TT]}$ has initial distribution $\ppa_0\sim\trans_{\TT}$ and, for $0<s<\TT$, rate matrices
\begin{equation}
\label{eq:chain_intro_main}
\rate_{\x\y}(s)=\begin{cases}
\frac{\ru^i_{\TT-s}}{\Ru^i_{\TT-s}}\,\denpos^i(\x,\TT-s) \qquad &\textnormal{if}\qquad \y=\x-e_i,\qquad i\in [\dd]\\
\frac{\rdi_{\TT-s}^i}{\Rdi^i_{\TT-s}}\, \denbin^i(\x,\TT-s)\qquad &\textnormal{if}\qquad \y=\x+e_i,\qquad i\in [\dd]\\
0 \qquad & \textnormal{if} \qquad \y\notin\{\x, \x\pm e_1,\ldots, \x\pm e_{\dd}\}\\
 -\sum_{\y\in \{\x\pm e_1,\ldots,\x\pm e_{\dd}\}}\rate_{\x\y}(s)\qquad &\textnormal{if}\qquad \y=\x
\end{cases}.
\end{equation}
In particular,
\begin{equation}
\label{eq:markov_proj_intro_main}
\ppa_s\sim\trans_{\TT-s}\qquad\textnormal{for all} \qquad s\in [0,\TT].
\end{equation}
\end{theorem}
The proof of \cref{thm:main} relies on combining discrete Tweedie's formulas for \emph{both} the Binomial and Poisson distributions, and is contained in \cref{app:binomial_poisson_process}-\cref{app:poisson_poisson_process}. 

While \cref{eq:chain_intro_main} writes the rates in terms of two latent denoisers, it follows from \cref{cor:latentupdown_binposs}-\cref{cor:latentupdown_possposs} that, at interior times, the rates can in fact be expressed using a \Def{single} latent denoiser, since  $\denpos$ and $\denbin$ can be expressed in terms of each other. However, this relation is \Def{non-linear} and \Def{non-local}, creating both numerical stability issues and more expensive neural function evaluations, which is why we opt for the formulation in \cref{eq:chain_intro_main}. Lastly, \cref{cor:latentupdown_binposs}-\cref{cor:latentupdown_possposs} also show that the data denoiser $\dend$ can be written as a non-linear and non-local function of a single latent denoiser.

\paragraph{Comparing the processes} Both the Binomial-Poisson process (for distributions over $\N^{\dd}$) and the Poisson-Poisson process (for distributions over $\Z^{\dd}$) are defined in terms of \texttt{latent-up} and \texttt{latent-down} denoisers. However, in the Binomial-Poisson process  the latent variable $\lrv_t$ depends on the data $\data$, while in the Poisson-Poisson process $(\lrv_t,\data)$ are independent. This is due to the fact that the  Binomial-Poisson process needs to remain in $\N^{\dd}$. 

An additional difference is that the final time distribution of the Binomial-Poisson perturbation \eqref{eq:perturb_intro} is exactly $ \pos(\Ru_{\TT})$, while the  final time distribution of the Poisson-Poisson perturbation \eqref{eq:perturb_pp_intro} is only approximately $\ske(\Rdi_{\TT},\Ru_{\TT})$. 

Finally, we note that while from a theoretical perspective the Poisson-Poisson process is easier to work with, in our experiments we often find the Binomial-Poisson process to be more effective.

\subsection{Training and sampling}

For ease of exposition, we present the algorithms using $\TT=1$, with $\Rdi_t=t$ for Binomial-Poisson, $\Rdi_t=bt$ for Poisson-Poisson, and $\Ru_t=at$ for both, where $a,b>0$. We omit any conditioning context, which would otherwise be passed to the neural networks, as well as additional clipping to the data range.

\paragraph{Training} We can apply two different approaches to training $\denbin$ and $\denpos$.

\textbf{\cref{alg:bin_poss_train}.} Use the fact that the optimal denoiser is, by definition, a conditional expectation so it can be trained using any valid Bregman divergence, in particular the mean-square error (MSE) loss.

\textbf{\cref{alg:bin_poss_train_ce}.} For each coordinate $i\in [\dd]$ train a marginal posterior distribution $\trans_{\data^i|\pp_t}$ and then use either \cref{lem:cor_cond_exp_binpos_posterior} (Binomial-Poisson) or \cref{lem:cor_cond_exp_pospos_posterior} (Poisson-Poisson) to compute $\denbin$ and $\denpos$. Crucially, note that we only need to train the collection of $\dd$ \emph{marginal} posteriors $\trans_{\data^i|\pp_t}$, rather than the full posterior $\trans_{\data|\pp_t}$ which is much more complicated. This is possible thanks to \cref{thm:main} which shows that ultimately we only need the denoisers in order to sample from $\target$.   
\begin{figure*}[!t]
\centering
\begin{minipage}[t]{0.485\textwidth}
\vspace{0pt}
\idlmalgcaption{alg:bin_poss_train}
{\cref{alg:bin_poss_train}: Denoiser Training}{}
\begin{lstlisting}[
    language=IDLMPseudocode,
    escapechar=`,
    basicstyle=\ttfamily\scriptsize,
    breaklines=true,
    gobble=4
]
    `\algcom{Process: Binomial-Poisson or Poisson-Poisson}`
    `\algcom{Noise scales $a,b>0$}`
    `\algcom{\algfn{L_bregman}: any valid Bregman divergence}`
    `\algcom{\algfn{sample_time} has full support on $(0,1)$}`
    x_data = `\algfn{sample_data}`(p_data)
    t = `\algfn{sample_time}`(0, 1)
    L_t = `\algfn{Binomial}`(x_data, t)`\algcom{Binomial-Poisson}`
    `\algcom{or \texttt{L\_t = }\algfn{Poisson}\texttt{(b * t)} for Poisson-Poisson}`
    N_t = `\algfn{Poisson}`(a * t)
    X_t = x_data - L_t + N_t
    m_L, m_N = `\algid{net_latents}`(X_t, t)
    loss = `\algfn{L_bregman}`(L_t, m_L)
         + `\algfn{L_bregman}`(N_t, m_N)
    `\algfn{update}`(`\algid{net_latents}`, loss)
\end{lstlisting}
\idlmalgbottom
\end{minipage}
\hfill
\begin{minipage}[t]{0.485\textwidth}
\vspace{0pt}
\idlmalgcaption{alg:bin_poss_train_ce}
{\cref{alg:bin_poss_train_ce}:  Marginal Posterior Training}{}
\begin{lstlisting}[
    language=IDLMPseudocode,
    escapechar=`,
    basicstyle=\ttfamily\scriptsize,
    breaklines=true,
    gobble=4
]
    `\algcom{Process: Binomial-Poisson or Poisson-Poisson}`
    `\algcom{Noise scales $a,b>0$}`
    `\algcom{\algfn{sample_time} has full support on $(0,1)$}`
    x_data = `\algfn{sample_data}`(p_data)
    t = `\algfn{sample_time}`(0, 1)
    L_t = `\algfn{Binomial}`(x_data, t)`\algcom{Binomial-Poisson}`
    `\algcom{or \texttt{L\_t = }\algfn{Poisson}\texttt{(b * t)} for Poisson-Poisson}`
    N_t = `\algfn{Poisson}`(a * t)
    X_t = x_data - L_t + N_t
    posterior = `\algid{net_posterior}`(X_t, t)
    for i = 1, ..., D:
        loss_i = `\algfn{cross_entropy}`(
            posterior_i, x_data^i)
    loss = sum_i(loss_i)
    `\algfn{update}`(`\algid{net_posterior}`, loss)
\end{lstlisting}
\idlmalgbottom
\end{minipage}
\end{figure*}
\paragraph{Sampling} The crux of \cref{thm:main} is that once the latent denoisers $(\denbin,\denpos)$ are trained, we can approximately sample from $\target$ using any sampling algorithm for CTMCs, with the initial distribution $\pos(a)$ (Binomial-Poisson) or $\ske(b,a)$ (Poisson-Poisson). The Poisson initialization matches the terminal distribution exactly, whereas the Skellam initialization introduces an approximation. For example, \cref{alg:tau} and \cref{alg:euler} present the tau-leaping and Euler methods, respectively. 
Here, \texttt{latents\_denoiser} returns estimates of $(\denbin,\denpos)$: it calls \texttt{net\_latents}, trained by regression (e.g., MSE), or converts the marginal posteriors returned by \texttt{net\_posterior} using the formulas in \cref{lem:cor_cond_exp_binpos_posterior,lem:cor_cond_exp_pospos_posterior}.
\begin{figure*}[!t]
\centering

\begin{minipage}[t]{0.485\textwidth}
\vspace{0pt}
\idlmalgcaption{alg:tau}
{\cref{alg:tau}: Tau-leaping sampling}{}
\begin{lstlisting}[
    language=IDLMPseudocode,
    escapechar=`,
    basicstyle=\ttfamily\scriptsize,
    breaklines=true,
    gobble=4
]
    `\algcom{Process: Binomial-Poisson or Poisson-Poisson}`
    `\algcom{Latent-mean estimator \algfn{latents_denoiser}; $a,b>0$}`
    `\algcom{Number of steps T; stopping time $0<\varepsilon<1$}`
    h = (1 - epsilon) / T
    x = `\algfn{Poisson}`(a)
    if process == "Poisson-Poisson":
        x = x - `\algfn{Poisson}`(b)
    for j = 0, ..., T - 1:
        t = 1 - j * h
        m_L, m_N = `\algfn{latents_denoiser}`(x, t)
        `\algcom{Apply coordinatewise for every $i\in [\dd]$}`
        lambda_+^i = `\algfn{max}`(m_L^i, 0) / t
        lambda_-^i = `\algfn{max}`(m_N^i, 0) / t
        P_+^i = `\algfn{Poisson}`(h * lambda_+^i)
        P_-^i = `\algfn{Poisson}`(h * lambda_-^i)
        if process == "Binomial-Poisson":
            P_-^i = `\algfn{min}`(P_-^i, x^i)
        x^i = x^i + P_+^i - P_-^i
    return x
\end{lstlisting}
\idlmalgbottom
\end{minipage}
\hfill
\begin{minipage}[t]{0.485\textwidth}
\vspace{0pt}
\idlmalgcaption{alg:euler}
{\cref{alg:euler}: Euler sampling}{}
\begin{lstlisting}[
    language=IDLMPseudocode,
    escapechar=`,
    basicstyle=\ttfamily\scriptsize,
    breaklines=true,
    gobble=4
]
    `\algcom{Process: Binomial-Poisson or Poisson-Poisson}`
    `\algcom{Latent-mean estimator \algfn{latents_denoiser}; $a,b>0$}`
    `\algcom{Number of steps T; stopping time $0<\varepsilon<1$}`
    h = (1 - epsilon) / T
    x = `\algfn{Poisson}`(a)
    if process == "Poisson-Poisson":
        x = x - `\algfn{Poisson}`(b)
    for j = 0, ..., T - 1:
        t = 1 - j * h
        m_L, m_N = `\algfn{latents_denoiser}`(x, t)
        `\algcom{Apply coordinatewise for every $i\in [\dd]$}`
        lambda_+^i = `\algfn{max}`(m_L^i, 0) / t
        lambda_-^i = `\algfn{max}`(m_N^i, 0) / t
        p_i = `\algfn{min}`(h * lambda_+^i, 1)
        q_i = `\algfn{min}`(h * lambda_-^i, 1)
        if process == "Binomial-Poisson":
            if x^i == 0:
                q_i = 0
        s_i = `\algfn{max}`(1, p_i + q_i)
        Delta^i = `\algfn{Categorical}`(
            +1 : p_i / s_i,
            -1 : q_i / s_i,
             0 : 1 - (p_i + q_i) / s_i)
        x^i = x^i + Delta^i
    return x
\end{lstlisting}
\idlmalgbottom
\end{minipage}
\end{figure*}

\section{Experiments}
\label{sec:exp}
We validate our method on synthetic datasets (\cref{sec:main_synthetic}), CIFAR-10 (\cref{sec:cifar10}), and Sudoku completion (\cref{sec:sudoku}). The code to reproduce all of the experiments is available at the anonymous link~\url{https://www.github.com/AnonrsySubmitter/JUD}. 

\subsection{Synthetic Experiments} \label{sec:main_synthetic}

We evaluate the performance of the Binomial-Poisson and Poisson-Poisson models (with \cref{alg:bin_poss_train}) on a series of benchmark low-dimensional examples where we estimate the Wasserstein-1 distance between the true and generated distributions. For each example, we use $N = 20,000$ samples from the true distribution to learn the denoisers. %
At evaluation time, each trained model generates $1{,}000$ samples via a reverse-time Euler discretization with $2{,}000$ steps. We compute $W_1$ between these samples and an independent batch of $1{,}000$ samples from the target distribution. Additional details on the data, model and training can be found in Appendix~\ref{app:synthetic}. 

\paragraph{Results.} Table~\ref{tab:w1-results} presents the Wasserstein-1 distance between the generated samples and true target samples for all models. Across almost all distributions, in particular those that have multiple modes (e.g., Poisson Mixture and Negative Binomial Mixture), we observe closer agreement between each target distribution and the generated samples with our proposed Binomial-Poisson process over alternative discrete diffusion models. In Figure~\ref{fig:synthetic_histogram}, we include a histogram and show close agreement with the Binomial-Poisson process to the tested distributions.

\begin{table}[htbp]
\centering
\caption{Wasserstein-1 distance across distributions and sampling methods (mean $\pm$ std). Our process is denoted in bold and the best results are highlighted in orange.}
\label{tab:w1-results}
\sisetup{detect-weight, detect-family}
\resizebox{\textwidth}{!}{%
\begin{tabular}{l
 S[table-format=2.2] @{${}\pm{}$} S[table-format=1.2] 
 S[table-format=2.2] @{${}\pm{}$} S[table-format=1.2]
  S[table-format=2.2] @{${}\pm{}$} S[table-format=1.2]
  S[table-format=2.2] @{${}\pm{}$} S[table-format=1.2]
  S[table-format=2.2] @{${}\pm{}$} S[table-format=1.2]
  S[table-format=2.2] @{${}\pm{}$} S[table-format=1.2]
}
\toprule
& \multicolumn{2}{c}{Binomial-only} & \multicolumn{2}{c}{Binomial-Poisson} & \multicolumn{2}{c}{Poisson-Poisson} & \multicolumn{2}{c}{Count Bridges} & \multicolumn{2}{c}{count-FM} & \multicolumn{2}{c}{SEDD} \\
\cmidrule(lr){2-3}\cmidrule(lr){4-5}\cmidrule(lr){6-7}\cmidrule(lr){8-9}\cmidrule(lr){10-11}\cmidrule(lr){12-13}
Distribution & \multicolumn{2}{c}{\footnotesize \cite{shenfeld2026binomial}} & \multicolumn{2}{c}{\footnotesize \textbf{(Ours)}} & \multicolumn{2}{c}{\footnotesize \textbf{(Ours)}} & \multicolumn{2}{c}{\footnotesize \cite{fishman2026count}} & \multicolumn{2}{c}{\footnotesize \cite{wei2026flow}} & \multicolumn{2}{c}{\footnotesize \cite{lou2024discrete}} \\
\midrule
Poisson                     & 0.31  & 0.05 & {\color{myhighlight} \bfseries 0.13}  & {\color{myhighlight} \bfseries 0.02} & 0.35  & 0.02 & 0.22 & 0.02 & 0.28 & 0.06 & 0.46 & 0.08 \\
Poisson Mixture             & 1.52  & 0.21 & {\color{myhighlight} \bfseries 1.39}  & {\color{myhighlight} \bfseries 0.30} & 89.71 & 0.34 & 1.69 & 0.34 & 1.68 & 0.41 & 1.96 & 0.33 \\
Zero-Inflated Poisson       & 0.20  & 0.03 & {\color{myhighlight} \bfseries 0.12}  & {\color{myhighlight} \bfseries 0.01} & 0.32  & 0.02 & 0.16 & 0.03 & 0.33 & 0.06 & 0.53 & 0.08 \\
Negative Binomial Mixture   & 1.68  & 0.23 & {\color{myhighlight} \bfseries 1.62}  & {\color{myhighlight} \bfseries 0.56} & 17.73 & 0.41 & 1.80 & 0.74 & 2.45 & 0.46 & 2.35 & 0.33 \\
Beta Negative Binomial      & 0.87  & 0.07 & {\color{myhighlight} \bfseries 0.65}  & {\color{myhighlight} \bfseries 0.06} & 6.14  & 0.14 & 0.73 & 0.07 & 1.52 & 0.14 & 2.57 & 0.14 \\
Zipf & {\color{myhighlight} 0.36}  & {\color{myhighlight} 0.02} & {\bfseries 0.39} & {\bfseries 0.03} & 1.79  & 0.10 & 0.64 & 0.05 & 0.66 & 0.09 & 1.08 & 0.09 \\
Yule--Simon                 & 0.32  & 0.04 & {\color{myhighlight} \bfseries 0.12}  & {\color{myhighlight} \bfseries 0.01} & 0.23  & 0.02 & 0.45 & 0.03 & 0.38 & 0.04 & 0.63 & 0.05 \\
Correlated Beta-Binomial    & 0.81  & 0.11 & {\bfseries 0.62}  & {\bfseries 0.10} & 11.39 & 0.19 & {\color{myhighlight} 0.59} & {\color{myhighlight} 0.09} & 0.63 & 0.04 & 1.23 & 0.38 \\
\bottomrule
\end{tabular}%
}
\end{table}

\begin{figure}
    \centering
    {\scriptsize \hspace{1cm} \textbf{Poisson} \hspace{2.5cm} \textbf{Poisson Mixture} \hspace{1cm} \textbf{Zero Inflated Poisson} \hspace{1cm}         \textbf{Negative Binomial Mixture}} \\
    \includegraphics[width=0.24\linewidth]{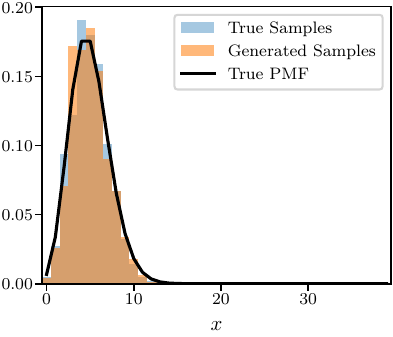}
    \includegraphics[width=0.24\linewidth]{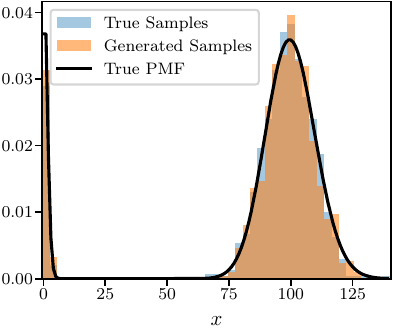}
    \includegraphics[width=0.24\linewidth]{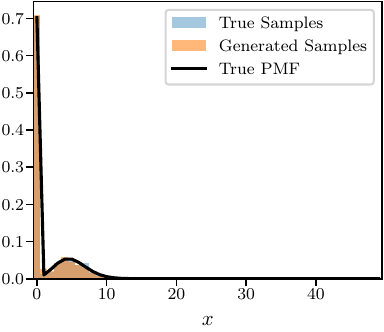}
    \includegraphics[width=0.24\linewidth]{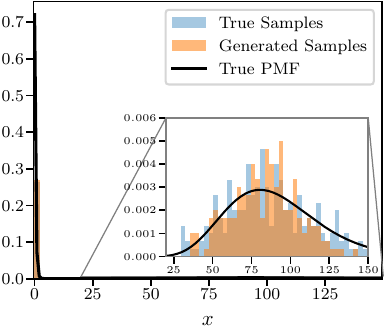} \\

    {\scriptsize \hspace{0.5cm} \textbf{Beta Negative Binomial} \hspace{1.8cm} \textbf{Zipf} \hspace{2.5cm} \textbf{Yule Simon} \hspace{1cm} \textbf{Correlated Beta Binomial}} \\
    \includegraphics[width=0.24\linewidth]{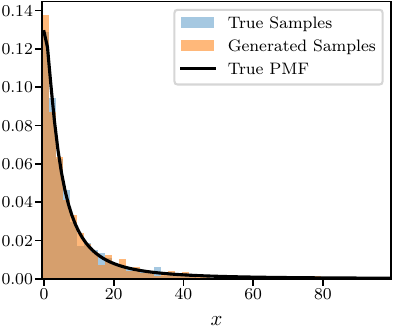}
    \includegraphics[width=0.24\linewidth]{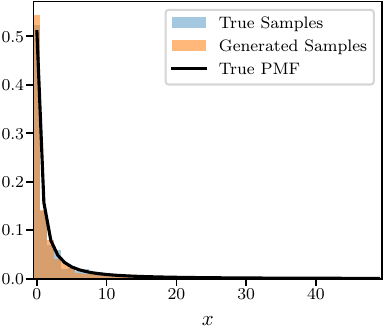}
    \includegraphics[width=0.24\linewidth]{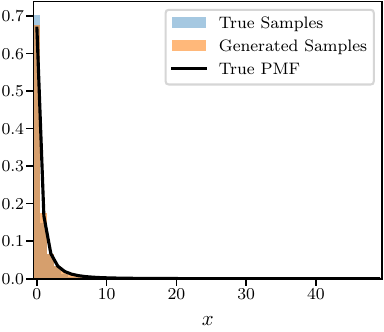}
    \includegraphics[width=0.24\linewidth]{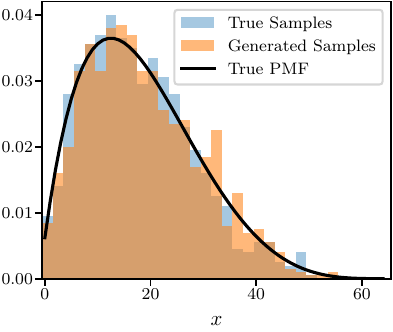}
    \vspace{-0.2cm}
    \caption{Target dataset and the generated samples using the Binomial-Poisson process for all eight synthetic distributions.  \label{fig:synthetic_histogram}}
\end{figure}

\subsection{CIFAR-10}
\label{sec:cifar10}
We evaluate the Binomial-Poisson model of \cref{thm:main} on unconditional generation of $32\times32$ CIFAR-10 images~\citep{cifar_krizhevsky}. We use the DDPM++ architecture~\citep{song2021scorebased} with an EDM-style preconditioning~\citep{karras2022elucidating}, regressing the targets $(\lrv_t, \possrv_t)$; see \cref{alg:bin_poss_train}. Both predictions are produced by a single shared network with two heads. Sampling uses tau-leaping with 2048 uniformly spaced steps; see \cref{alg:tau}. Training and sampling details are given in \cref{app:cifar10_training,app:cifar10_sampling}.

We achieve a FID of $2.07$ on 50,000 generated images. \Cref{tab:cifar10_fid} compares generation quality with continuous, hybrid, and discrete diffusion models. We obtain the lowest FID among the discrete models in this comparison. \Cref{fig:cifar10_samples} shows the first 64 randomly sampled images from the trained model.

\begin{table}[!ht]
\centering
\begin{minipage}[c]{0.485\textwidth}
\centering
\caption{Unconditional CIFAR-10 generation. FID evaluated on 50,000 generated samples.}
\label{tab:cifar10_fid}
\small
\renewcommand{\arraystretch}{1.08}
\begin{tabular}{lr}
\toprule
Method & FID $\downarrow$ \\
\midrule
\multicolumn{2}{l}{\textit{Continuous data models}} \\
EDM \citep{karras2022elucidating} & 1.97 \\
EDM (our reproduction) & 2.09 \\
ScoreSDE \citep{song2021scorebased} & 2.20 \\
IDDPM \citep{pmlr-v139-nichol21a} & 2.90 \\
DDPM \citep{ho2020denoising} & 3.17 \\
\midrule
\multicolumn{2}{l}{\textit{Hybrid data models}} \\
CADD\textsuperscript{$\dagger$} \citep{zheng2025cadd} & 2.88 \\
\midrule
\multicolumn{2}{l}{\textit{Discrete (ordered) data models}} \\
\textbf{\jud: Binomial-Poisson (ours)} & \textbf{2.07} \\
Binomial Flow \citep{shenfeld2026binomial} & 2.94 \\
Blackout \citep{santos2023blackout} & 4.58 \\
LTJ \citep{chen2023learning} & 4.80 \\
ItDPDM \citep{DBLP:journals/corr/abs-2505-05082} & 4.84 \\
\midrule
\multicolumn{2}{l}{\textit{Discrete (categorical) data models}} \\
MDM-Prime \citep{chao2025mdmprime} & 3.26 \\
Discrete FM \citep{gat2024discrete} & 3.63 \\
$\tau$-LDR-10 \citep{DBLP:journals/corr/abs-2205-14987} & 3.74 \\
MDM \citep{chao2025mdmprime} & 4.66 \\
MDM-Mixture \citep{chao2025mdmprime} & 4.80 \\
D3PM (Gauss.) \citep{austin2021structured} & 7.34 \\
CTDD-DG \citep{nisonoff2025unlocking} & 7.86 \\
\bottomrule
\end{tabular}

\smallskip
{\footnotesize $\dagger$ CADD augments discrete diffusion with a paired continuous latent diffusion.}
\end{minipage}\hfill
\begin{minipage}[c]{0.485\textwidth}
\centering
\includegraphics[width=\linewidth]{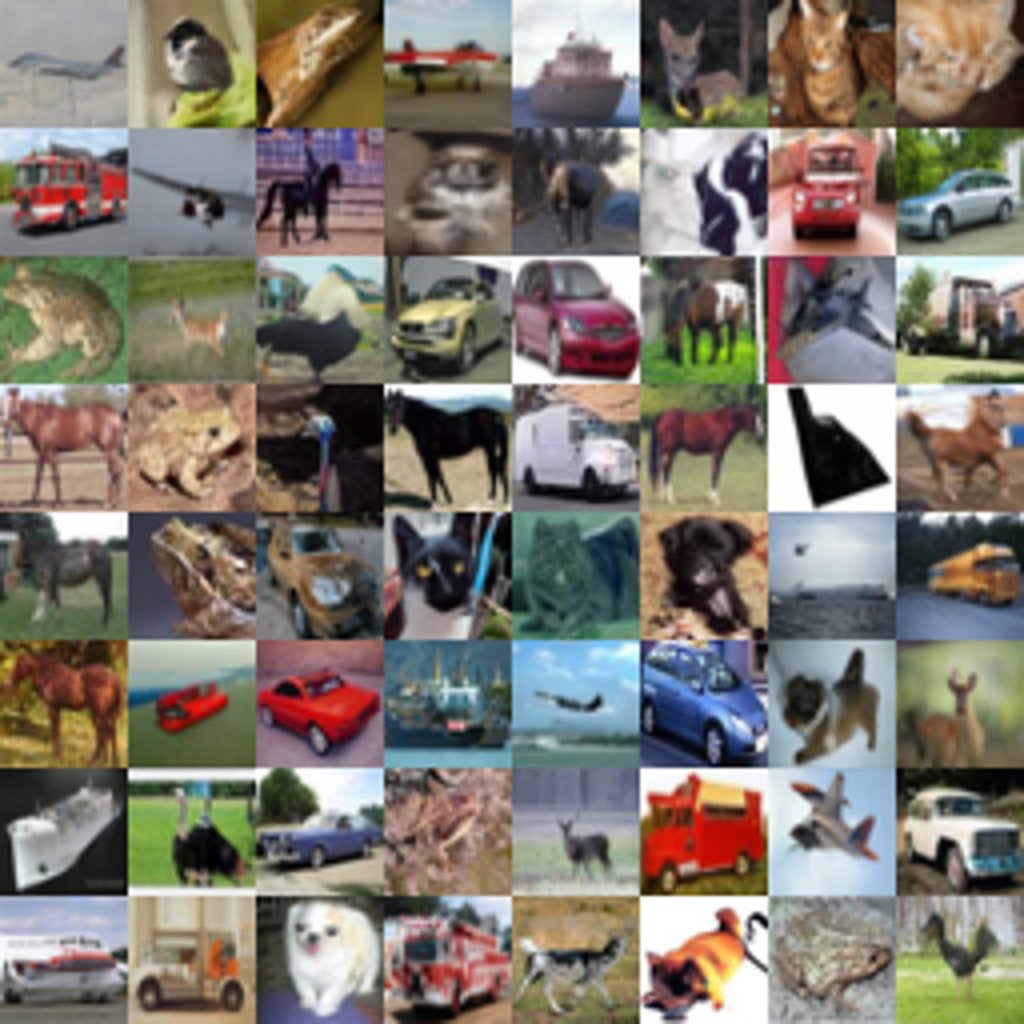}
\captionof{figure}{Unconditional CIFAR-10 samples from the Binomial-Poisson model.}
\label{fig:cifar10_samples}
\end{minipage}
\end{table}

\Cref{app:cifar10_sampling_ablations} reports FID as a function of the number of sampling steps for both samplers; at 512 steps the FID is 2.29.

\subsection{Sudoku}
\label{sec:sudoku}
We test the  Binomial-Poisson model of \cref{thm:main} on the hard Sudoku benchmark of \citet{deschenaux2026hyperspherical}. Each puzzle has 30 visible digits and a unique solution. The dataset comprises 48,000 training and 2,000 validation puzzles. The metric is exact-board accuracy: all 81 predicted digits must match the reference solution. Note that this dataset is not ordinal since Sudoku is invariant under permutations of the symbols. Our goal in this task is to explore the capabilities of our framework to general categorical data. 

\paragraph{Model and evaluation.}
We use the benchmark's DiT backbone configuration. We use the same linear-time schedules as in CIFAR-10, with mean digit value 5. We train with two types of loss: \textbf{(1)} mean-square error (MSE) and \textbf{(2)} cross entropy (CE). 
MSE uses the same $(\lrv_t, \possrv_t)$ targets and affine preconditioning as CIFAR-10 (\cref{tab:edm_generic,tab:edm_denoisers}); see \cref{alg:bin_poss_train}. CE instead learns a nine-class distribution for each digit, conditioned on the perturbed solution and context; see \cref{alg:bin_poss_train_ce}. The denoiser(s) follow analytically from this posterior (\cref{app:sudoku_posterior}). A single network evaluation returns two scalar outputs per cell for MSE or nine logits for CE. Training configurations are given in \cref{app:sudoku_training}. 

Sampling uses the tau-leaping scheme (\cref{alg:tau}) on a uniform grid from $t=1$ to $t=10^{-3}$. At the final $t$, if training with MSE, we round and clip $\pp_t+\denbin(\pp_t,t)-\denpos(\pp_t,t)$
 to $\{1,\ldots,9\}$, whereas if we train with CE we take the most probable digit under the posterior. We use 180 network evaluations, including the final evaluation, matching the diffusion and flow baselines.

\paragraph{Results.}
\Cref{tab:sudoku_hard} compares our results with the published baselines. A separate block reports digit-shuffled training, which applies a random permutation of $1,\ldots,9$ jointly to each puzzle. This preserves Sudoku constraints while varying the assignment of symbols to ordered counts. Posterior CE achieves higher mean accuracy than MSE in both settings.

\begin{table}[htbp]
\centering
\caption{Hard Sudoku exact-board accuracy. Baselines are from Table~1 of \citet{deschenaux2026hyperspherical}. Our scores are means (not best) over eight independent training runs.}
\label{tab:sudoku_hard}
\small
\renewcommand{\arraystretch}{1.08}
\begin{tabular}{lr}
\toprule
Method & Accuracy (\%) $\uparrow$ \\
\midrule
\multicolumn{2}{l}{\textit{Autoregressive models}} \\
AR (sampled) & 0.6 \\
AR (greedy) & 1.0 \\
\midrule
\multicolumn{2}{l}{\textit{Discrete diffusion models}} \\
MDLM & 30.2 \\
Duo & 58.4 \\
\midrule
\multicolumn{2}{l}{\textit{Continuous and hybrid models}} \\
CANDI & 16.7 \\
CANDI (pure Gaussian) & 12.5 \\
FLM & 44.5 \\
$\mathbb{S}$-FLM & 14.0 \\
$\mathbb{S}$-FLM + $\alpha^\star(0.1)$\textsuperscript{$\dagger$} & 43.2 \\
$\mathbb{S}$-FLM + $\alpha^\star(0.1)$ + adaptive\textsuperscript{$\dagger$} & 45.0 \\
\midrule
\multicolumn{2}{l}{\textbf{Binomial-Poisson (ours)}} \\
\textbf{MSE (noise/lost count)} & \textbf{17.49} \\
\textbf{Posterior CE} & \textbf{40.31} \\
\midrule
\multicolumn{2}{l}{\textbf{Binomial-Poisson (ours)}, \textit{with digit shuffling}} \\
\textbf{MSE (noise/lost count)} & \textbf{29.00} \\
\textbf{Posterior CE} & \textbf{66.99} \\
\bottomrule
\end{tabular}

\smallskip
{\footnotesize $\dagger$ Schedule variants: $\alpha^\star(0.1)$ denotes truncation, with an additional adaptive noise schedule where indicated.}
\end{table}

\section{Conclusion}
In this work we introduce \Def{Jumping Up and Down (\jud)}: a  \Def{denoiser-based} family of generative diffusion models for discrete ordinal data (either $\N^{\dd}$ or $\Z^{\dd}$) which allows for \Def{bi-directional perturbations}. Our main result (\cref{thm:main}) shows that it suffices to learn \Def{latent denoisers} in order to sample from the data distribution. Further, we show that theoretically a \Def{single} denoiser suffices to sample from the data distribution, as the two latent denoisers can be expressed in terms of each other. In addition, we also prove that the \Def{data denoiser} can be expressed in terms of the latent denoiser(s). We validate our methods on synthetic datasets, CIFAR-10, and the Sudoku completion task. In particular, on synthetic datasets we often obtain superior results over alternative discrete diffusion models, and we achieve a state-of-the-art result among discrete diffusion models for CIFAR-10.

\section{Limitations}
(1) For the purpose of generative modeling considered in this work, the source distribution should be easy to sample (which is the case for us, as it is simply a Poisson  or approximately Skellam). But in some applications  (e.g., \citet{fishman2026count}) we would like to have an arbitrary source measure, which our method does not allow for. (2) Currently, our method requires a high number of neural function evaluations (NFE) for sampling.  (3) Compared to the Binomial-Poisson method, the Poisson-Poisson method has an extra parameter which needs to be tuned, namely the final time $\TT$, in order to make the approximation $\transp_{\TT}\overset{\textnormal{approx.}}{\sim} \ske(\Rdi_{\TT},\Ru_{\TT})$ more accurate. 

\subsection*{Acknowledgements}

This material is based upon work supported by the National Science Foundation under Award Numbers DMS-2331920 and DMS-2508545. RB is grateful for the support from the Natural Sciences and Engineering Research Council of Canada Discovery (award number RGPIN-2026-07896).

\subsection*{AI use statement}

In this work, we used generative AI tools to assist in the writing of proofs, specifically to carry out or verify calculations in individual proof steps (such as the EDM scalings in \cref{sec:edm}), and to implement methods, specifically to help write the code for the experiments. We have not used generative AI tools to generate synthetic datasets, to develop the methodology or the theoretical framework, to formulate mathematical claims, to devise proof strategies, to propose or refine hypotheses, to design the experiments, or to interpret results; translation, dataset cleaning and reformatting, and qualitative or thematic data analysis are not applicable to this work. Additionally, we used generative AI tools to edit the paper for readability. We have reviewed all AI-assisted work: every AI-assisted calculation was checked by the authors, and all AI-assisted code was reviewed and tested by the authors. We take responsibility for the final content of this work, including text, claims or artifacts produced with the aid of generative AI.

\subsection*{Ethics statement}
Our work is methodological in nature and uses public datasets. We declare no ethical issues corresponding to human subjects, potential conflicts of interest and sponsorship, discrimination/bias/fairness concerns, privacy and security issues, legal
compliance, and research integrity issues. While generative AI tools can cause harm, our work in this context is at the research scale.

\subsection*{Reproducibility statement}
We provide extensive details on our experiments in the paper itself and the appendices. We also provide an anonymized link to the code producing the experiments: ~\url{https://www.github.com/AnonrsySubmitter/JUD}.

\bibliography{jump_up_down}
\bibliographystyle{iclr2027_conference}

\appendix

\section{Poisson and Binomial distributions}
\label{app:binomial_poisson}

In this section we derive some properties of the Poisson and Binomial distributions which we will repeatedly use. To ease the notation we set $[\dd]:=\{1,\ldots,\dd\}$  and recall that $\{e_i\}_{i=1}^{\dd}$ is the standard basis. We denote
\begin{equation}
\label{eq:notation_poss_bin}
\poss_{\alpha}:=\pos(\alpha)\qquad\textnormal{and}\qquad \binn_{\n,r}:=\bin(\n;r),
\end{equation}
for $\alpha\in \R_{\ge 0}^{\dd}$, $\n\in \N^{\dd}$, and $r\in [0,1]^{\dd}$, i.e.,
\begin{equation}
\label{eq:poss_alpha}
\poss_{\alpha}(\x)=\prod_{i=1}^{\dd}\poss_{\alpha^i}(\x^i)=\prod_{i=1}^{\dd}e^{-\alpha^i}\frac{(\alpha^i)^{\x^i}}{\x^i!},
\end{equation}
and
\begin{equation}
\label{eq:bin_nr}
\binn_{\n,r}(\x)=\prod_{i=1}^{\dd}\binn_{\n^i,r^i}(\x^i)=\prod_{i=1}^{\dd}1_{\x^i\le \n^i}{\n^i\choose \x^i}(r^i)^{\x^i}(1-r^i)^{\n^i-\x^i}.
\end{equation}
We will use the convention $\poss_{\alpha}(\x)=0$ for $\x\notin\N^{\dd}$ and $\binn_{\n,r}(\x)=0$ if $x^i\notin \{0,\ldots, n^i\} $ for some $i\in [\dd]$. 
\begin{lemma}[Properties of Poisson distribution]
Let $[0,\infty)\ni t\mapsto \alpha_t\in \R^{\dd}_{\ge 0}$ be differentiable. Then, at times when $\alpha_t\in\R_{>0}^{\dd}$,
\begin{equation}
\label{eq:poss_pde}
\partial_t\poss_{\alpha_t}(\x)=\poss_{\alpha_t}(\x)\sum_{i=1}^{\dd}\left(\x^i\frac{\dot{\alpha}^i_t}{\alpha^i_t}-\dot{\alpha}^i_t\right),\qquad \x\in \N^{\dd},
\end{equation}
and  
\begin{equation}
\label{eq:poss_idn}
\poss_{\alpha_t}(\x-e_i)=\frac{\x^i}{\alpha^i_t}\poss_{\alpha_t}(\x),\qquad \x\in \N^{\dd},\qquad i\in [\dd].
\end{equation}
Moreover, for all $t\ge 0$,
\begin{equation}
\label{eq:poss_pde_idn}
\partial_t\poss_{\alpha_t}(\x)=\sum_{i=1}^{\dd}\dot{\alpha}^i_t\,(\poss_{\alpha_t}(\x-e_i)-\poss_{\alpha_t}(\x)),\qquad \x\in \N^{\dd}.
\end{equation}
\end{lemma}

\begin{proof}
\cref{eq:poss_pde} follows from
\begin{align*}
\partial_t\log\poss_{\alpha_t}(\x)=\partial_t\sum_{i=1}^{\dd}\log\left(e^{-\alpha^i_t}\frac{(\alpha^i_t)^{\x^i}}{\x^i!}\right)=\sum_{i=1}^{\dd}\left(\x^i\frac{\dot{\alpha}^i_t}{\alpha^i_t}-\dot{\alpha}^i_t\right),
\end{align*}
while \cref{eq:poss_idn} follows from
\begin{align*}
\poss_{\alpha_t}(\x-e_i)=e^{-\alpha^i_t}\frac{(\alpha^i_t)^{(\x^i-1)}}{(\x^i-1)!}\prod_{j\neq i}\poss_{\alpha^j_t}(\x^j)=\frac{\x^i}{\alpha^i_t}\poss_{\alpha^i_t}(\x^i)\prod_{j\neq i}\poss_{\alpha^j_t}(\x^j).
\end{align*}
\cref{eq:poss_pde_idn} follows from the combination of \eqref{eq:poss_pde} and \eqref{eq:poss_idn}. At zero parameter values, it follows by direct differentiation of the probability mass. 
\end{proof}

\begin{lemma}[Properties of Binomial distribution]
Fix $\n\in \N^{\dd}$ and let $[0,\infty)\ni t\mapsto r_t\in [0,1]^{\dd}$ be differentiable. Then, at times when $r_t\in(0,1)^{\dd}$,
\begin{equation}
\label{eq:bin_pde}
\partial_t \binn_{\n,r_t}(\x)= \binn_{\n,r_t}(\x)\sum_{i=1}^{\dd}\left(\x^i\frac{\dot{r}_t^i}{r^i_t}-(\n^i-\x^i)\frac{\dot{r}^i_t}{1-r^i_t}\right),\qquad \x\in \N^{\dd},
\end{equation}
and
\begin{equation}
\label{eq:bin_idn}
 \binn_{\n,r_t}(\x+e_i)=\frac{\n^i-\x^i}{\x^i+1}\frac{r^i_t}{1-r^i_t} \,\binn_{\n,r_t}(\x),\qquad \x\in \N^{\dd},\qquad i\in [\dd].
\end{equation}
Consequently, 
\begin{equation}
\label{eq:bin_pde_idn}
\partial_t \binn_{\n,r_t}(\x)= \sum_{i=1}^{\dd}\frac{\dot{r}_t^i}{r^i_t}\left(\x^i\binn_{\n,r_t}(\x)-(\x^i+1)\binn_{\n,r_t}(\x+e_i)\right),\qquad \x\in \N^{\dd}.
\end{equation}
\end{lemma}
\begin{proof}
\cref{eq:bin_pde} follows from 
\begin{align*}
&\partial_t\log \binn_{\n,r_t}(\x)=\partial_t\sum_{i=1}^{\dd}\log\left(1_{\x^i\le \n^i}{\n^i\choose \x^i}(r_t^i)^{\x^i}(1-r_t^i)^{\n^i-\x^i}\right)\\
&=\partial_t\sum_{i=1}^{\dd}\left[\x^i\log (r_t^i)+(\n^i-\x^i)\log (1-r_t^i)\right]=\sum_{i=1}^{\dd}\left(\x^i\frac{\dot{r}_t^i}{r_t^i}-(\n^i-\x^i)\frac{\dot{r}_t^i}{1-r_t^i}\right),
\end{align*}
while \cref{eq:bin_idn} follows from 
\begin{align*}
\binn_{\n,r_t}(\x+e_i)&=1_{\x^i+1\le \n^i}{\n^i\choose \x^i+1}(r^i_t)^{\x^i+1}(1-r^i_t)^{\n^i-\x^i-1}\prod_{j\neq i}\binn_{\n^j,r^j_t}(\x^j) \\
&=1_{\x^i+1\le \n^i}\frac{\n^i-\x^i}{\x^i+1}\frac{r^i_t}{1-r^i_t}\binn_{\n^i,r^i_t}(\x^i)\prod_{j\neq i}^{\dd}\binn_{\n^j,r^j_t}(\x^j).
\end{align*}
\cref{eq:bin_pde_idn} follows from the combination of \eqref{eq:bin_pde} and  \eqref{eq:bin_idn}. 
\end{proof}

\begin{lemma}[Convolution of Binomial and Poisson]
Fix $\n\in \N^{\dd}$ and let $[0,\infty)\ni t\mapsto r_t\in [0,1]^{\dd}$ and $[0,\infty)\ni t\mapsto \alpha_t\in \R^{\dd}_{\ge 0}$ be differentiable. Then, for $\x\in\N^{\dd}$ and times when $r_t\in(0,1)^{\dd}$ and $\alpha_t\in\R_{>0}^{\dd}$, 
\begin{equation}
\begin{split}
\label{eq:conv_bin_poss_partial_t}
&\partial_t(\binn_{\n,r_t}\ast \poss_{\alpha_t})(\x)\\
&=\sum_{i=1}^{\dd}\left(\frac{\dot{r}^{\idx}_t}{r^i_t(1-r^i_t)}\x^i-\dot{\alpha}^i_t-\frac{\dot{r}^{\idx}_t}{1-r^i_t}\n^{\idx}\right)(\binn_{\n,r_t}\ast \poss_{\alpha_t})(\x)\\
&-\sum_{i=1}^{\dd}\left(\frac{\dot{r}^{\idx}_t}{r^i_t(1-r^i_t)}-\frac{\dot{\alpha}^i_t}{\alpha^i_t}\right)\alpha^i_t\, (\binn_{\n,r_t}\ast \poss_{\alpha_t})(\x-e_i). 
\end{split}
\end{equation}
\end{lemma}
\begin{proof}
We start with a one-dimensional convolution. For fixed $i\in [\dd]$, 
\begin{align*}
&\partial_t(\binn_{\n^{\idx},r^{\idx}_t}\ast \poss_{\alpha^{\idx}_t})(\x^{\idx})=\partial_t\sum_{\kk^i=0}^{\n^{\idx}}\binn_{\n^{\idx},r^{\idx}_t}(\kk^i)\,\poss_{\alpha^{\idx}_t}(\x^{\idx}-\kk^i)\\
&=\sum_{\kk^i=0}^{\n^{\idx}}\left\{\poss_{\alpha^{\idx}_t}(\x^{\idx}-\kk^i)\,\binn_{\n^{\idx},r^{\idx}_t}(\kk^i)\left(\kk^i\frac{\dot{r}^{\idx}_t}{r^i_t}-(\n^i-\kk^i)\frac{\dot{r}^i_t}{1-r^i_t}\right)\right\}\\
&+\sum_{\kk^i=0}^{\n^{\idx}}\left\{\poss_{\alpha^{\idx}_t}(\x^{\idx}-\kk^i)\,\binn_{\n^{\idx},r^{\idx}_t}(\kk^i)\left((\x^i-\kk^i)\frac{\dot{\alpha}^i_t}{\alpha^i_t}-\dot{\alpha}^i_t\right)\right\}\\
&=\left(\frac{\dot{\alpha}^i_t}{\alpha^i_t}\x^{\idx}-\dot{\alpha}^i_t-\n^{\idx}\frac{\dot{r}^{\idx}_t}{1-r^i_t}\right)(\binn_{\n^{\idx},r^{\idx}_t}\ast \poss_{\alpha^{\idx}_t})(\x^{\idx})\\
&+\left(\frac{\dot{r}^{\idx}_t}{r^i_t(1-r^i_t)}-\frac{\dot{\alpha}^i_t}{\alpha^i_t}\right)\sum_{\kk^i=0}^{\n^{\idx}}\kk^i\,\binn_{\n^{\idx},r^{\idx}_t}(\kk^i)\,\poss_{\alpha^{\idx}_t}(\x^{\idx}-\kk^i).
\end{align*}
By \cref{eq:poss_idn},
\begin{align*}
&\sum_{\kk^i=0}^{\n^{\idx}}\kk^i\,\binn_{\n^{\idx},r^{\idx}_t}(\kk^i)\,\poss_{\alpha^{\idx}_t}(\x^{\idx}-\kk^i)\\
&=\x^i\sum_{\kk^i=0}^{\n^{\idx}}\binn_{\n^{\idx},r^{\idx}_t}(\kk^i)\,\poss_{\alpha^{\idx}_t}(\x^{\idx}-\kk^i)-\sum_{\kk^i=0}^{\n^{\idx}}(\x^i-\kk^i)\,\poss_{\alpha^{\idx}_t}(\x^{\idx}-\kk^i)\binn_{\n^{\idx},r^{\idx}_t}(\kk^i)\\
&=\x^i(\binn_{\n^{\idx},r^{\idx}_t}\ast \poss_{\alpha^{\idx}_t})(\x^{\idx})-\alpha^i_t\sum_{\kk^i=0}^{\n^{\idx}}\poss_{\alpha^{\idx}_t}(\x^{\idx}-1-\kk^i)\,\binn_{\n^{\idx},r^{\idx}_t}(\kk^i)\\
&=\x^i(\binn_{\n^{\idx},r^{\idx}_t}\ast \poss_{\alpha^{\idx}_t})(\x^{\idx})-\alpha^i_t\,(\binn_{\n^{\idx},r^{\idx}_t}\ast \poss_{\alpha^{\idx}_t})(\x^{\idx}-1),
\end{align*}
so we conclude 
\begin{align*}
&\partial_t(\binn_{\n^{\idx},r^{\idx}_t}\ast \poss_{\alpha^{\idx}_t})(\x^{\idx})=\\
&\left(\frac{\dot{r}^{\idx}_t}{r^i_t(1-r^i_t)}\x^i-\dot{\alpha}^i_t-\frac{\dot{r}^{\idx}_t}{1-r^i_t}\n^{\idx}\right)(\binn_{\n^{\idx},r^{\idx}_t}\ast \poss_{\alpha^{\idx}_t})(\x^{\idx})-\left(\frac{\dot{r}^{\idx}_t}{r^i_t(1-r^i_t)}-\frac{\dot{\alpha}^i_t}{\alpha^i_t}\right)\alpha^i_t\,(\binn_{\n^{\idx},r^{\idx}_t}\ast \poss_{\alpha^{\idx}_t})(\x^{\idx}-1).
\end{align*}
It follows that 
\begin{align*}
&\partial_t(\binn_{\n,r_t}\ast \poss_{\alpha_t})(\x)=\sum_{\idx=1}^{\dd}\left(\prod_{j\neq \idx}^{\dd}(\binn_{\n^j,r^j_t}\ast \poss_{\alpha^j_t})(\x^j)\right)\partial_t(\binn_{\n^{\idx},r^{\idx}_t}\ast \poss_{\alpha^{\idx}_t})(\x^{\idx})\\
&=\sum_{i=1}^{\dd}\left(\frac{\dot{r}^{\idx}_t}{r^i_t(1-r^i_t)}\x^i-\dot{\alpha}^i_t-\frac{\dot{r}^{\idx}_t}{1-r^i_t}\n^{\idx}\right)(\binn_{\n,r_t}\ast \poss_{\alpha_t})(\x)\\
&-\sum_{i=1}^{\dd}\left(\frac{\dot{r}^{\idx}_t}{r^i_t(1-r^i_t)}-\frac{\dot{\alpha}^i_t}{\alpha^i_t}\right)\alpha^i_t\,(\binn_{\n,r_t}\ast \poss_{\alpha_t})(\x-e_i).
\end{align*}
\end{proof}
\section{The Binomial-Poisson process}
\label{app:binomial_poisson_process}
The main result of this section is the Binomial-Poisson part of \cref{thm:main}:
\begin{theorem}[The Binomial-Poisson process]
\label{thm:main_intro}
Under the schedule assumptions of \cref{thm:main}, set the noise schedules $\Rdi_t\in [0,1]^{\dd}, \Ru_t\in \R_{\ge 0}^{\dd}$ to satisfy
\begin{equation}
\label{eq:binpossrates}
\Rdi^i_0=0,\quad\Rdi^i_{\TT}=1, \quad\Ru^i_0=0,  \qquad \ru^i_t:=\frac{\diff}{\diff t}\Ru^i_t,\quad \rdi^i_t:=\frac{\diff}{\diff t}\Rdi^i_t,\qquad i\in [\dd].
\end{equation}
Let $(\trans_t)_{t\in [0,\TT]}$ be the marginals of $(\pp_t)_{t\in [0,\TT]}$,
\begin{equation}
\label{eq:perturb_intro_thm}  
\pp_t:=\data-\lrv_t+\possrv_t,\qquad\textnormal{with} \qquad \lrv_t\sim\bin(\data,\Rdi_t),\quad \possrv_t\sim \pos(\Ru_t),
\end{equation} 
where $\possrv_t$ is independent of $(\data,\lrv_t)$, and define the latent denoisers 
\begin{equation}
\label{eq:denoisers_bin_poss_thm}
\denbin(\x,t):=\EE[\lrv_t|\pp_t=\x],\qquad \denpos(\x,t):=\EE[\possrv_t|\pp_t=\x].
\end{equation}
Write $s=\TT-t$ for reverse time. The time reversal $(\ppa_s)_{s\in [0,\TT]}$ of the forward CTMC in \cref{prop:aux_markov} has initial distribution $\ppa_0\sim\trans_{\TT}=\pos(\Ru_{\TT})$ and, for $0<s<\TT$, rate matrices
\begin{equation}
\label{eq:chain_intro}
\rate_{\x\y}(s)=\begin{cases}
\frac{\ru^i_{\TT-s}}{\Ru^i_{\TT-s}}\,\denpos^i(\x,\TT-s) \qquad &\textnormal{if}\qquad \y=\x-e_i,\qquad i\in [\dd]\\
\frac{\rdi_{\TT-s}^i}{\Rdi^i_{\TT-s}}\, \denbin^i(\x,\TT-s)\qquad &\textnormal{if}\qquad \y=\x+e_i,\qquad i\in [\dd]\\
0 \qquad & \textnormal{if} \qquad \y\notin\{\x, \x\pm e_1,\ldots, \x\pm e_{\dd}\}\\
 -\sum_{\y\in \{\x\pm e_1,\ldots,\x\pm e_{\dd}\}}\rate_{\x\y}(s)\qquad &\textnormal{if}\qquad \y=\x
\end{cases}.
\end{equation}
In particular,
\begin{equation}
\label{eq:markov_proj_intro}
\ppa_s\sim\trans_{\TT-s}\qquad\textnormal{for all} \qquad s\in [0,\TT].
\end{equation}
\end{theorem}
We start by simplifying the notation and write
\begin{equation}
\label{eq:reparm_bin}
\binnrv_t:=\data-\lrv_t\qquad \text{and}\qquad \Rd_t:=1-\Rdi_t,
\end{equation}
so in this section we focus on the perturbations of $\data\sim\target$, for a given $t\in [0,\TT]$, by
\begin{equation}
\label{eq:perturb_app}  
\pp_t:=\binnrv_t+\possrv_t,\qquad\textnormal{with} \qquad \binnrv_t\sim\bin(\data,\Rd_t),\quad \possrv_t\sim \pos(\Ru_t),
\end{equation} 
where $\possrv_t$ is independent of $(\data,\binnrv_t)$. We will use
\begin{equation}
\label{eq:denoiser_def_app_intro}
\den(\x,t):=\EE[\binnrv_t|\pp_t=\x]=\x-\EE[\possrv_t|\pp_t=\x],
\end{equation}
and note that we have the relations
\begin{align}
&\denbin(\x,t)=\dend(\x,t)-\den(\x,t)\label{eq:latents_data_latents},\\
&\denpos(\x,t)=\x-\den(\x,t).\label{eq:latents_data_latents2}
\end{align}

We begin by recording some of the conditional distributions of $(\data,\binnrv_t,\possrv_t)$. 
\begin{lemma}
\label{lem:conditional_dist}
Let $\trans_t$ be the distribution of $\pp_t$ given in \eqref{eq:perturb_app}. 
For $\x\in \N^{\dd}$,
\begin{equation}
\label{eq:marg_exp_bin_poss}
\trans_t(\x)=\sum_{\x_0\in \N^{\dd}}\target(\x_0)(\binn_{\x_0,\Rd_t}\ast \poss_{\Ru_t})(\x). 
\end{equation}
Fix $\x\in \N^{\dd}$, then, for $\y\in  \N^{\dd}$, 
\begin{equation}
\label{eq:Bt_given_Xt}
\trans_{\binnrv_t|\pp_t=\x}(\y)=\frac{\poss_{\Ru_t}(\x -\y)}{{\trans_t(\x)}}\sum_{\x_0\in\N^{\dd}}\target(\x_0)\binn_{\x_0,\Rd_t}(\y).
\end{equation}
Fix $\x,\z\in \N^{\dd}$, then, for $\y\in  \N^{\dd}$, 
\begin{align}
&\trans_{\binnrv_t|\pp_t=\x,\data=\z }(\y)=\frac{\binn_{\z,\Rd_t}(\y)\poss_{\Ru_t}(\x-\y)}{(\binn_{\z,\Rd_t}\ast \poss_{\Ru_t})(\x)
},\label{eq:Bt_given_XtNt}\\
&\trans_{\lrv_t|\pp_t=\x,\data=\z }(\y)=\frac{\binn_{\z,\Rdi_t}(\y)\poss_{\Ru_t}(\x+\y-\z)}{\sum_{\tilde{\y}\in \N^{\dd}} \binn_{\z,\Rdi_t}(\tilde{\y})\poss_{\Ru_t}(\x+\tilde{\y}-\z)},\label{eq:Lt_given_Xt}\\
&\trans_{\possrv_t|\pp_t=\x,\data=\z }(\y)=\frac{\binn_{\z,\Rdi_t}(\z+\y-\x)\poss_{\Ru_t}(\y)}{\sum_{\tilde{\y}\in \N^{\dd}} \binn_{\z,\Rdi_t}(\z+\tilde{\y}-\x)\poss_{\Ru_t}(\tilde{\y})}.\label{eq:Nt_given_Xt}
\end{align}

Fix $\y\in \N^{\dd}$, then, for $\x\in  \N^{\dd}$, 
\begin{equation}
\label{eq:data_given_Xt}
\trans_{\data|\pp_t=\y}(\x)=\frac{\target(\x)}{\trans_t(\y)}(\binn_{\x,\Rd_t}\ast \poss_{\Ru_t})(\y).
\end{equation}
\end{lemma}
\begin{proof}
\cref{eq:marg_exp_bin_poss} holds by the definition \eqref{eq:perturb_app}. For \cref{eq:Bt_given_Xt}, since $\pp_t|\binnrv_t\sim \poss_{\Ru_t}(\cdot -\binnrv_t)$, Bayes' rule gives
\[
\trans_{\binnrv_t|\pp_t=\x}(\y)=\frac{\trans_{\binnrv_t}(\y)}{\trans_{\pp_t}(\x)}\trans_{\pp_t|\binnrv_t=\y}(\x)=\frac{\sum_{\x_0\in\N^{\dd}}\target(\x_0)\binn_{\x_0,\Rd_t}(\y)}{{\trans_t(\x)}}\poss_{\Ru_t}(\x -\y).
\]
For \cref{eq:Bt_given_XtNt} we use Bayes' rule and the independence of $(\data,\binnrv_t)$ and $\possrv_t$ to write 
\begin{align*}
&\trans_{\binnrv_t|\pp_t=\x,\data=\z }(\y)=\frac{\trans_{\binnrv_t,\pp_t,\data}(\y,\x,\z)}{\trans_{\pp_t,\data}(\x,\z)}=\frac{\trans_{\binnrv_t,\possrv_t,\data}(\y,\x-\y,\z)}{\trans_{\pp_t,\data}(\x,\z)}\\
&=\frac{\trans_{\binnrv_t,\data}(\y,\z)\trans_{\possrv_t}(\x-\y)}{\trans_{\pp_t,\data}(\x,\z)}=\frac{\target(\z)\trans_{\binnrv_t|\data=\z}(\y)\trans_{\possrv_t}(\x-\y)}{\target(\z)\trans_{\pp_t|\data=\z}(\x)}\\
&=\frac{\binn_{\z,\Rd_t}(\y)\poss_{\Ru_t}(\x-\y)}{(\binn_{\z,\Rd_t}\ast \poss_{\Ru_t})(\x)
}.
\end{align*}
Similarly, for \cref{eq:Lt_given_Xt} and \cref{eq:Nt_given_Xt} we have 
\begin{align*}
&\trans_{\lrv_t|\pp_t=\x,\data=\z }(\y)=\frac{\trans_{\lrv_t,\pp_t,\data}(\y,\x,\z)}{\trans_{\pp_t,\data}(\x,\z)}=\frac{\trans_{\lrv_t,\possrv_t,\data}(\y,\x+\y-\z,\z)}{\trans_{\pp_t,\data}(\x,\z)}\\
&=\frac{\trans_{\lrv_t,\data}(\y,\z)\trans_{\possrv_t}(\x+\y-\z)}{\trans_{\pp_t,\data}(\x,\z)}=\frac{\target(\z)\binn_{\z,\Rdi_t}(\y)\poss_{\Ru_t}(\x+\y-\z)}{\target(\z)\trans_{\pp_t|\data=\z}(\x)}\\
&=\frac{\binn_{\z,\Rdi_t}(\y)\poss_{\Ru_t}(\x+\y-\z)}{\trans_{\pp_t|\data=\z}(\x)}=\frac{\binn_{\z,\Rdi_t}(\y)\poss_{\Ru_t}(\x+\y-\z)}{\sum_{\tilde{\y}\in \N^{\dd}} \binn_{\z,\Rdi_t}(\tilde{\y})\poss_{\Ru_t}(\x+\tilde{\y}-\z)},
\end{align*}
and
\begin{align*}
&\trans_{\possrv_t|\pp_t=\x,\data=\z }(\y)=\frac{\trans_{\possrv_t,\pp_t,\data}(\y,\x,\z)}{\trans_{\pp_t,\data}(\x,\z)}=\frac{\trans_{\possrv_t,\lrv_t,\data}(\y,\z+\y-\x,\z)}{\trans_{\pp_t,\data}(\x,\z)}\\
&=\frac{\trans_{\lrv_t,\data}(\z+\y-\x,\z)\trans_{\possrv_t}(\y)}{\trans_{\pp_t,\data}(\x,\z)}=\frac{\target(\z)\binn_{\z,\Rdi_t}(\z+\y-\x)\poss_{\Ru_t}(\y)}{\target(\z)\trans_{\pp_t|\data=\z}(\x)}\\
&=\frac{\binn_{\z,\Rdi_t}(\z+\y-\x)\poss_{\Ru_t}(\y)}{\trans_{\pp_t|\data=\z}(\x)}=\frac{\binn_{\z,\Rdi_t}(\z+\y-\x)\poss_{\Ru_t}(\y)}{\sum_{\tilde{\y}\in \N^{\dd}} \binn_{\z,\Rdi_t}(\z+\tilde{\y}-\x)\poss_{\Ru_t}(\tilde{\y})}.
\end{align*}
For \cref{eq:data_given_Xt} we also use Bayes' rule,
\begin{align*}
\trans_{\data|\pp_t=\y}(\x)=\frac{\target(\x)}{\trans_{\pp_t}(\y)}\trans_{\pp_t|\data=\x}(\y)=\frac{\target(\x)}{\trans_t(\y)}(\binn_{\x,\Rd_t}\ast \poss_{\Ru_t})(\y).
\end{align*}
\end{proof}

\subsection{Discrete Tweedie's formula}
\label{app:tweedie}
In this section we derive a discrete Tweedie's formula for the Binomial-Poisson perturbation \eqref{eq:perturb_app}. This formula is crucial in proving \cref{thm:main_intro}. 
\begin{lemma}[Discrete Tweedie's formula]
\label{lem{tweedie_bin_poss}}
Fix $t\in (0,\TT)$, and let $\pp_t:=\binnrv_t+\possrv_t$ where 
\begin{equation}
\label{eq:up_down_pertb_lem}
\data\sim \target,\quad \binnrv_t|\data\sim \bin(\data,\Rd_t),\quad \possrv_t\sim \pos(\Ru_t),
\end{equation}
with $\possrv_t$ independent of $(\data,\binnrv_t)$. Denote by $\trans_t$ the distribution of $\pp_t$, i.e., $\pp_t\sim \trans_t$, and define 
\begin{equation}
\label{eq:denoiser_def_lem}
\den(\x,t):=\EE[\binnrv_t|\pp_t=\x]=\x-\EE[\possrv_t|\pp_t=\x].
\end{equation}
Then,
\begin{equation}
\label{eq:tweedie}
    \frac{\trans_t(\x-e_{\idx})}{\trans_t(\x)}=\frac{\x^{\idx}-\den^i(\x,t)}{\Ru^i_t},\qquad \x\in\N^{\dd},\quad t\in (0,\TT),\quad i\in [\dd].
\end{equation}
\end{lemma}
\begin{proof}
\begin{align*}
&\den^i(\x,t)=\sum_{\y\in\N^{\dd}}    \y^i\trans_{\binnrv_t|\pp_t=\x}(\y)\overset{\ref{eq:Bt_given_Xt}}{=}\frac{1}{{\trans_t(\x)}}\sum_{\x_0\in\N^{\dd}}\target(\x_0)\sum_{\y\in\N^{\dd}}\binn_{\x_0,\Rd_t}(\y)\,\poss_{\Ru_t}(\x -\y)\,\y^i\\
&=\x^i-\frac{1}{{\trans_t(\x)}}\sum_{\x_0\in\N^{\dd}}\target(\x_0)\sum_{\y\in\N^{\dd}}\binn_{\x_0,\Rd_t}(\y)\,\poss_{\Ru_t}(\x -\y)\,(\x^i-\y^i)\\
&\overset{\ref{eq:poss_idn}}{=}\x^i-\frac{\Ru^i_t}{{\trans_t(\x)}}\sum_{\x_0\in\N^{\dd}}\target(\x_0)\sum_{\y\in\N^{\dd}}\binn_{\x_0,\Rd_t}(\y)\,\poss_{\Ru_t}(\x -\y-e_i)\\
&=\x^i-\Ru^i_t\,\frac{\trans_t(\x-e_i)}{\trans_t(\x)}\frac{1}{\trans_t(\x-e_i)}\sum_{\x_0\in\N^{\dd}}\target(\x_0)\sum_{\y\in\N^{\dd}}\binn_{\x_0,\Rd_t}(\y)\,\poss_{\Ru_t}(\x-e_i-\y)\\
&\overset{\ref{eq:Bt_given_Xt}}{=}\x^i-\Ru^i_t\,\frac{\trans_t(\x-e_i)}{\trans_t(\x)}\sum_{\y\in \N^{\dd}}\trans_{\binnrv_t|\pp_t=\x-e_i}(\y)=\x^i-\Ru^i_t\,\frac{\trans_t(\x-e_i)}{\trans_t(\x)}.
\end{align*}
\end{proof}
The next result will help us show how to relate the \texttt{latent-up},  \texttt{latent-down}, and \texttt{data} denoisers. 
\begin{lemma}[Relation between \texttt{data} and \texttt{latent} denoisers]
\label{lem:den-dendd}
For $\x\in\N^{\dd}, t\in (0,\TT)$, and $i\in [\dd]$,
\begin{align}
\dend^i(\x,t)&=\den^i(\x,t)+\frac{1-\Rd^i_t}{\Rd^i_t}\frac{\trans_t(\x+e_i)}{\trans_t(\x)}\den^i(\x+e_i,t)\label{eq:den-dendd1}\\
&=\den^i(\x,t)+\frac{\Ru_t^i(1-\Rd_t^i)}{\Rd_t^i}\frac{\den^i(\x+e_i,t)}{\x^i+1-\den^i(\x+e_i,t)}\label{eq:den-dendd2}.
\end{align}
\end{lemma}

\begin{proof}
\begin{align*}
&\dend^i(\x,t)=\EE[\data^i|\pp_t=\x]=\sum_{\x_0\in\N^{\dd}}\x_0^i\,\trans_{\data|\pp_t=\x}(\x_0)\\
&\overset{\ref{eq:data_given_Xt}}{=}\frac{1}{\trans_t(\x)}\sum_{\x_0\in\N^{\dd}}\x_0^i\,(\binn_{\x_0,\Rd_t}\ast \poss_{\Ru_t})(\x)\,\target(\x_0)\\
&=\frac{1}{\trans_t(\x)}\sum_{\x_0\in\N^{\dd}}\sum_{\n\in\N^{\dd}}\x_0^i\,\binn_{\x_0,\Rd_t}(\n)\,\poss_{\Ru_t}(\x-\n)\,\target(\x_0)\\
&\overset{\ref{eq:bin_idn}}{=}\frac{1}{\trans_t(\x)}\sum_{\x_0\in\N^{\dd}}\sum_{\n\in\N^{\dd}}\left(\frac{1-\Rd^i_t}{\Rd^i_t}(\n^i+1)\binn_{\x_0,\Rd_t}(\n+e_i)+\n^i\binn_{\x_0,\Rd_t}(\n)\right)\poss_{\Ru_t}(\x-\n)\,\target(\x_0)\\
&=\frac{1-\Rd^i_t}{\Rd^i_t}\frac{\trans_t(\x+e_i)}{\trans_t(\x)}\sum_{\n\in\N^{\dd}}(\n^i+1)\frac{\sum_{\x_0\in\N^{\dd}}\target(\x_0)\binn_{\x_0,\Rd_t}(\n+e_i)}{\trans_t(\x+e_i)}\poss_{\Ru_t}(\x+e_i-(\n+e_i))\\
&+\sum_{\n\in\N^{\dd}}\n^i\frac{\sum_{\x_0\in\N^{\dd}}\target(\x_0)\binn_{\x_0,\Rd_t}(\n)}{\trans_t(\x)}\poss_{\Ru_t}(\x-\n)\\
&\overset{\ref{eq:Bt_given_Xt}}{=}\frac{1-\Rd^i_t}{\Rd^i_t}\frac{\trans_t(\x+e_i)}{\trans_t(\x)}\sum_{\n\in\N^{\dd}}(\n^i+1)\,\trans_{\binnrv_t|\pp_t=\x+e_i}(\n+e_i)+\sum_{\n\in\N^{\dd}}\n^i\,\trans_{\binnrv_t|\pp_t=\x}(\n)\\
&\overset{\textnormal{re-index}}{=}\frac{1-\Rd^i_t}{\Rd^i_t}\frac{\trans_t(\x+e_i)}{\trans_t(\x)}\sum_{\n\in\N^{\dd}}\n^i\,\trans_{\binnrv_t|\pp_t=\x+e_i}(\n)+\sum_{\n\in\N^{\dd}}\n^i\,\trans_{\binnrv_t|\pp_t=\x}(\n)\\
&=\frac{1-\Rd^i_t}{\Rd^i_t}\frac{\trans_t(\x+e_i)}{\trans_t(\x)}\den^i(\x+e_i,t)+\den^i(\x,t)\\
&\overset{\ref{eq:tweedie}}{=}\frac{\Ru^i_t(1-\Rd^i_t)}{\Rd^i_t}\frac{\den^i(\x+e_i,t)}{\x^{\idx}+1-\den^i(\x+e_i,t)}+\den^i(\x,t).
\end{align*}
\end{proof}
We can now show that the \texttt{latent-up} and \texttt{latent-down} denoisers can be expressed in terms of each other, in a non-linear and non-local way. Consequently, the \texttt{data} denoiser can be expressed in terms of a single latent denoiser,  in a non-linear and non-local way.
\begin{corollary}[Relations between \texttt{latent-up}, \texttt{latent-down}, and \texttt{data} denoisers]
\label{cor:latentupdown_binposs}
Fix $\x\in\N^{\dd},t\in (0,\TT)$ and $i\in [\dd]$. Then,
\begin{equation}
\label{eq:tweedie_poss_bin_up}
\denbin^i(\x,t)=\frac{\Ru_t^i(1-\Rd_t^i)}{\Rd_t^i}\frac{\x^i+1-\denpos^i(\x+e_i,t)
}{\denpos^i(\x+e_i,t)},
\end{equation} 
and 
\begin{align}
\dend^i(\x,t)&=\x^i-\denpos^i(\x,t)+\frac{\Ru_t^i(1-\Rd_t^i)}{\Rd_t^i}\frac{\x^i+1-\denpos^i(\x+e_i,t)
}{\denpos^i(\x+e_i,t)}.\label{eq:den-dendd2_bin}
\end{align}
\end{corollary}
\begin{proof}
We have
\begin{align*}
&\denpos^i(\x,t)\overset{\ref{eq:latents_data_latents2}}{=}\x^i-\den^i(\x,t)\overset{\ref{eq:latents_data_latents}}{=}  \x^i +\denbin^i(\x,t)-\dend^i(\x,t)\\
&\overset{\ref{eq:den-dendd2}}{=}\x^i +\denbin^i(\x,t)-\left(\den^i(\x,t)+\frac{\Ru_t^i(1-\Rd_t^i)}{\Rd_t^i}\frac{\den^i(\x+e_i,t)}{\x^i+1-\den^i(\x+e_i,t)}\right)\\
&\overset{\ref{eq:latents_data_latents2}}{=}\x^i +\denbin^i(\x,t)-\left(\x^i-\denpos^i(\x,t)+\frac{\Ru_t^i(1-\Rd_t^i)}{\Rd_t^i}\frac{\x^i+1-\denpos^i(\x+e_i,t)
}{\denpos^i(\x+e_i,t)}\right),
\end{align*}
so \cref{eq:tweedie_poss_bin_up} follows. For \cref{eq:den-dendd2_bin}
we use
\[
\dend(\x,t)=\x+\denbin(\x,t)-\denpos(\x,t),
\]
together with \cref{eq:tweedie_poss_bin_up}.
\end{proof}
We conclude by expressing the latent denoisers in terms of the marginal posteriors of the data given its noisy observations. 
\begin{lemma}
\label{lem:cor_cond_exp_binpos_posterior}
Fix $i\in [\dd]$ and $\x\in \N^{\dd}$. Then,
\begin{align}
&\den^i(\x,t)=\sum_{\z^i\in \N}\trans_{\data^i|\pp_t=\x}(\z^i)\sum_{\y^i\in \N}\y^i\, \frac{\binn_{\z^i,\Rd_t^i}(\y^i)\poss_{\Ru_t^i}(\x^i-\y^i)}{(\binn_{\z^i,\Rd_t^i}\ast \poss_{\Ru_t^i})(\x^i) },\label{eq:cor_cond_exp_posterior}\\
&\denbin^i(\x,t)=\sum_{\z^i\in \N}\trans_{\data^i|\pp_t=\x}(\z^i)\sum_{\y^i\in \N}\y^i\,\frac{\binn_{\z^i,\Rdi_t^i}(\y^i)\poss_{\Ru_t^i}(\x^i+\y^i-\z^i)}{\sum_{w\in \N} \binn_{\z^i,\Rdi_t^i}(w)\poss_{\Ru_t^i}(\x^i+w-\z^i)},\label{eq:cor_cond_exp_posterior_L}\\
&\denpos^i(\x,t)=\sum_{\z^i\in \N}\trans_{\data^i|\pp_t=\x}(\z^i)\sum_{\y^i\in \N}\y^i\,\frac{\binn_{\z^i,\Rdi_t^i}(\z^i+\y^i-\x^i)\poss_{\Ru_t^i}(\y^i)}{\sum_{w\in \N} \binn_{\z^i,\Rdi_t^i}(\z^i+w-\x^i)\poss_{\Ru_t^i}(w)}.\label{eq:cor_cond_exp_posterior_N}
\end{align}
\end{lemma}
\begin{proof}
We first note that \cref{eq:Bt_given_XtNt}, \cref{eq:Lt_given_Xt}, and \cref{eq:Nt_given_Xt} show that $\trans_{\binnrv_t|\pp_t=\x,\data=\z }(\y)$, $\trans_{\lrv_t|\pp_t=\x,\data=\z }(\y)$, and $\trans_{\possrv_t|\pp_t=\x,\data=\z }(\y)$ are product distributions with marginals  
\begin{align}
&\trans_{\binnrv_t^i|\pp_t=\x,\data=\z }(\y^i)=\frac{\binn_{\z^i,\Rd_t^i}(\y^i)\poss_{\Ru_t^i}(\x^i-\y^i)}{(\binn_{\z^i,\Rd_t^i}\ast \poss_{\Ru_t^i})(\x^i)
},\label{eq:Bt_given_XtNt_i}\\
&\trans_{\lrv_t^i|\pp_t=\x,\data=\z}(\y^i)=\frac{\binn_{\z^i,\Rdi_t^i}(\y^i)\poss_{\Ru_t^i}(\x^i+\y^i-\z^i)}{\sum_{w\in \N} \binn_{\z^i,\Rdi_t^i}(w)\poss_{\Ru_t^i}(\x^i+w-\z^i)},\label{eq:Lt_given_Xt_i}\\
&\trans_{\possrv_t^i|\pp_t=\x,\data=\z }(\y^i)=\frac{\binn_{\z^i,\Rdi_t^i}(\z^i+\y^i-\x^i)\poss_{\Ru_t^i}(\y^i)}{\sum_{w\in \N} \binn_{\z^i,\Rdi_t^i}(\z^i+w-\x^i)\poss_{\Ru_t^i}(w)}\label{eq:Nt_given_Xt_i}.
\end{align}
By marginalization it follows that 
\begin{align*}
&\trans_{\binnrv_t^i|\pp_t=\x,\data^i=\z^i }(\y^i)=\sum_{\z^{-i}\in\N^{\dd-1}}\trans_{\binnrv_t^i|\pp_t=\x,\data^i=\z^i, \data^{-i}=\z^{-i}}(\y^i)\trans_{\data^{-i}|\pp_t=\x,\data^i=\z^i}(\z^{-i})\\
&=\sum_{\z^{-i}\in\N^{\dd-1}}\trans_{\binnrv_t^i|\pp_t=\x, \data=\z}(\y^i)\trans_{\data^{-i}|\pp_t=\x,\data^i=\z^i}(\z^{-i})\\
&=\trans_{\binnrv_t^i|\pp_t=\x, \data=\z}(\y^i)\underbrace{\sum_{\z^{-i}\in\N^{\dd-1}}\trans_{\data^{-i}|\pp_t=\x,\data^i=\z^i}(\z^{-i})}_{=1}=\trans_{\binnrv_t^i|\pp_t=\x, \data=\z}(\y^i),
\end{align*}
and similarly,
\begin{align*}
&\trans_{\lrv_t^i|\pp_t=\x,\data^i=\z^i }(\y^i)=\trans_{\lrv_t^i|\pp_t=\x, \data=\z}(\y^i),\\
&\trans_{\possrv_t^i|\pp_t=\x,\data^i=\z^i }(\y^i)=\trans_{\possrv_t^i|\pp_t=\x, \data=\z}(\y^i).
\end{align*}
Combining with \cref{eq:Bt_given_XtNt_i}, \cref{eq:Lt_given_Xt_i}, and \cref{eq:Nt_given_Xt_i} we conclude 
\begin{align}
&\trans_{\binnrv_t^i|\pp_t=\x,\data^i=\z^i }(\y^i)=\frac{\binn_{\z^i,\Rd_t^i}(\y^i)\poss_{\Ru_t^i}(\x^i-\y^i)}{(\binn_{\z^i,\Rd_t^i}\ast \poss_{\Ru_t^i})(\x^i)
},\label{eq:Bt_given_XtNt_ii}\\
&\trans_{\lrv_t^i|\pp_t=\x,\data^i=\z^i}(\y^i)=\frac{\binn_{\z^i,\Rdi_t^i}(\y^i)\poss_{\Ru_t^i}(\x^i+\y^i-\z^i)}{\sum_{w\in \N} \binn_{\z^i,\Rdi_t^i}(w)\poss_{\Ru_t^i}(\x^i+w-\z^i)},\label{eq:Lt_given_Xt_ii}\\
&\trans_{\possrv_t^i|\pp_t=\x,\data^i=\z^i }(\y^i)=\frac{\binn_{\z^i,\Rdi_t^i}(\z^i+\y^i-\x^i)\poss_{\Ru_t^i}(\y^i)}{\sum_{w\in \N} \binn_{\z^i,\Rdi_t^i}(\z^i+w-\x^i)\poss_{\Ru_t^i}(w)}\label{eq:Nt_given_Xt_ii}.
\end{align}
Hence, by Bayes' rule,
\begin{align*}
&\den^i(\x,t)=\EE[\binnrv_t^i|\pp_t=\x]=\sum_{\y^i\in \N}\y^i\, \trans_{\binnrv_t^i|\pp_t=\x}(\y^i)\\
&=\sum_{\y^i\in \N}\y^i\, \frac{\trans_{\binnrv_t^i,\pp_t}(\y^i,\x)}{\trans_{\pp_t}(\x)}=\sum_{\y^i\in \N}\y^i\, \frac{1}{\trans_{\pp_t}(\x)}\sum_{\z^i\in \N}\trans_{\binnrv_t^i,\pp_t, \data^i}(\y^i,\x,\z^i)\\
&=\sum_{\y^i\in \N}\y^i\, \sum_{\z^i\in \N}\trans_{\binnrv_t^i|\pp_t=\x, \data^i=\z^i}(\y^i) \frac{\trans_{\pp_t, \data^i}(\x,\z^i) }{\trans_{\pp_t}(\x)}\\
&=\sum_{\y^i\in \N}\y^i\, \sum_{\z^i\in \N}\trans_{\binnrv_t^i|\pp_t=\x, \data^i=\z^i}(\y^i) \trans_{\data^i|\pp_t=\x}(\z^i) \\
&\overset{\ref{eq:Bt_given_XtNt_ii}}{=}\sum_{\z^i\in \N}\trans_{\data^i|\pp_t=\x}(\z^i)\sum_{\y^i\in \N}\y^i\, \frac{\binn_{\z^i,\Rd_t^i}(\y^i)\poss_{\Ru_t^i}(\x^i-\y^i)}{(\binn_{\z^i,\Rd_t^i}\ast \poss_{\Ru_t^i})(\x^i) },
\end{align*}
which establishes \cref{eq:cor_cond_exp_posterior}. 
Similarly, using \cref{eq:Lt_given_Xt_ii} and \cref{eq:Nt_given_Xt_ii},  we get \cref{eq:cor_cond_exp_posterior_L} and \cref{eq:cor_cond_exp_posterior_N}.
\end{proof}
\subsection{Forward and backward processes}
\label{app:for_back}
In this section we will show how to use the latent denoisers in order to define forward and  backward processes with the marginals of the Binomial-Poisson perturbation \eqref{eq:perturb_app}. Consequently, we prove  \cref{thm:main_intro}. 

For $i\in [\dd]$ and $t\in [0,\TT]$ set
\begin{equation}
\label{eq:lam_Lam} 
\Ru^i_0=0,\quad \Rd^i_0=1,\quad\Rd^i_{\TT}=0, \qquad \ru^i_t:=\frac{\diff}{\diff t}\Ru^i_t,\quad \rd^i_t:=\frac{\diff}{\diff t}\Rd^i_t.
\end{equation}
We begin by deriving the evolution equation of the marginals of the Binomial-Poisson  perturbation \eqref{eq:perturb_app}.
\begin{lemma}
\label{lem:marg_bin_poss}
Let $\pp_t\sim\trans_t$ where $\pp_t$ is as in \cref{eq:perturb_app}. Then, for $\x\in \N^{\dd}$ and $t\in (0,\TT)$,
\begin{align}
&\partial_t\trans_t(\x)\nonumber\\
&=\sum_{\idx=1}^{\dd}\left(\frac{\rd^i_t}{\Rd^i_t(1-\Rd^i_t)}\x^{\idx}-\ru^i_t-\frac{\rd^i_t}{1-\Rd^i_t}\dend^i(\x,t)\right)\trans_t(\x)-\sum_{i=1}^{\dd}\left(\frac{\rd^{\idx}_t\Ru^i_t}{\Rd^i_t(1-\Rd^i_t)}-\ru^i_t\right)\trans_t(\x-e_i)\label{eq:partial_pt1}\\
&=\sum_{\idx=1}^{\dd}\left(\frac{\rd^i_t}{\Rd^i_t(1-\Rd^i_t)}\x^{\idx}-\ru^i_t-\frac{\rd^i_t}{1-\Rd^i_t}\den^i(\x,t)\right)\trans_t(\x)-\sum_{i=1}^{\dd}\left(\frac{\rd^{\idx}_t\Ru^i_t}{\Rd^i_t(1-\Rd^i_t)}-\ru^i_t\right)\trans_t(\x-e_i)\label{eq:partial_pt2}\\
&-\sum_{\idx=1}^{\dd}\left(\frac{\rd^i_t}{\Rd^i_t}\den^i(\x+e_i,t)\right)\trans_t(\x+e_i).\nonumber
\end{align}
\end{lemma}
\begin{proof}
\begin{align*}
&\partial_t\trans_t(\x)=\partial_t\sum_{\x_0\in \N^{\dd}}\target(\x_0)(\binn_{\x_0,\Rd_t}\ast \poss_{\Ru_t})(\x)\\
&\overset{\ref{eq:conv_bin_poss_partial_t}}{=}\sum_{\x_0\in \N^{\dd}}\target(\x_0)\left\{(\binn_{\x_0,\Rd_t}\ast \poss_{\Ru_t})(\x)\sum_{\idx=1}^{\dd}\left(\frac{\rd^i_t}{\Rd^i_t(1-\Rd^i_t)}\x^{\idx}-\ru^i_t-\frac{\rd^i_t}{1-\Rd^i_t}\x_0^{\idx}\right)\right\}\\
&-\sum_{\x_0\in \N^{\dd}}\target(\x_0)\left\{\sum_{i=1}^{\dd}\left(\frac{\rd^{\idx}_t\Ru^i_t}{\Rd^i_t(1-\Rd^i_t)}-\ru^i_t\right)(\binn_{\x_0,\Rd_t}\ast \poss_{\Ru_t})(\x-e_i)\right\}\\
&=\trans_t(\x)\sum_{\idx=1}^{\dd}\left(\frac{\rd^i_t}{\Rd^i_t(1-\Rd^i_t)}\x^{\idx}-\ru^i_t\right)-\sum_{\idx=1}^{\dd}\frac{\rd^i_t}{1-\Rd^i_t}\sum_{\x_0\in \N^{\dd}}\x_0^i\,\underbrace{\target(\x_0)(\binn_{\x_0,\Rd_t}\ast \poss_{\Ru_t})(\x)}_{=\trans_{0,t}(\x_0,\x)=\trans_{0|t}(\x_0|\x)\trans_t(\x)}\\
&-\sum_{i=1}^{\dd}\left(\frac{\rd^{\idx}_t\Ru^i_t}{\Rd^i_t(1-\Rd^i_t)}-\ru^i_t\right)\underbrace{\sum_{\x_0\in \N^{\dd}}\target(\x_0)(\binn_{\x_0,\Rd_t}\ast \poss_{\Ru_t})(\x-e_i)}_{\trans_t(\x-e_i)}\\
&=\trans_t(\x)\sum_{\idx=1}^{\dd}\left(\frac{\rd^i_t}{\Rd^i_t(1-\Rd^i_t)}\x^{\idx}-\ru^i_t-\frac{\rd^i_t}{1-\Rd^i_t}\sum_{\x_0\in \N^{\dd}}\x_0^i\,\trans_{0|t}(\x_0|\x)\right)-\sum_{i=1}^{\dd}\left(\frac{\rd^{\idx}_t\Ru^i_t}{\Rd^i_t(1-\Rd^i_t)}-\ru^i_t\right)\trans_t(\x-e_i)\\
&=\trans_t(\x)\sum_{\idx=1}^{\dd}\left(\frac{\rd^i_t}{\Rd^i_t(1-\Rd^i_t)}\x^{\idx}-\ru^i_t-\frac{\rd^i_t}{1-\Rd^i_t}\dend^i(\x,t)\right)-\sum_{i=1}^{\dd}\left(\frac{\rd^{\idx}_t\Ru^i_t}{\Rd^i_t(1-\Rd^i_t)}-\ru^i_t\right)\trans_t(\x-e_i),
\end{align*}
and we conclude as
\begin{align*}
&\trans_t(\x)\sum_{\idx=1}^{\dd}\left(\frac{\rd^i_t}{\Rd^i_t(1-\Rd^i_t)}\x^{\idx}-\ru^i_t-\frac{\rd^i_t}{1-\Rd^i_t}\dend^i(\x,t)\right)\\
&=\trans_t(\x)\sum_{\idx=1}^{\dd}\left(\frac{\rd^i_t}{\Rd^i_t(1-\Rd^i_t)}\x^{\idx}-\ru^i_t-\frac{\rd^i_t}{1-\Rd^i_t}\den^i(\x,t)\right)-\sum_{\idx=1}^{\dd}\left(\frac{\rd^i_t}{\Rd^i_t}\den^i(\x+e_i,t)\right)\trans_t(\x+e_i).
\end{align*}
\end{proof}
Next we will show how to use the denoiser \eqref{eq:denoiser_def_lem} to define a CTMC whose marginals are the same as the Binomial-Poisson perturbation \eqref{eq:perturb_app}.
\begin{proposition}
\label{prop:aux_markov}
Let $(\vec{\ppa_t})_{t\in [0,\TT]}$ be the CTMC with, for $0<t<\TT$, rate matrices
\begin{equation}
\label{eq:rate_axu_binposs}
\vec{\rate}_{\x\y}(t)=
\begin{cases}
\ru^i_t \qquad &\textnormal{if}\qquad \y=\x+e_i,\qquad i\in [\dd]\\
-\frac{\rd^i_t}{\Rd^i_t}\,\den^i(\x,t)\qquad &\textnormal{if}\qquad \y=\x-e_i,\qquad i\in [\dd]\\
0 \qquad & \textnormal{if} \qquad \y\notin\{\x, \x\pm e_1,\ldots, \x\pm e_{\dd}\}\\
 -\sum_{\y\in \{\x\pm e_1,\ldots,\x\pm e_{\dd}\}}\vec{\rate}_{\x\y}(t)\qquad &\textnormal{if}\qquad \y=\x
\end{cases},
\end{equation}
with $\vec{\ppa_0}\sim\target$. Then,
\[
\vec{\ppa_t}\sim \trans_t\qquad\textnormal{for all} \qquad t\in [0,\TT].
\]
\end{proposition}
\begin{proof}
It suffices to show that $(\trans_t)_{t\in [0,\TT]}$ satisfies the Kolmogorov equation
\begin{equation}
\label{eq:kolmogorov_to_prove}
\begin{split}
&\partial_t\trans_t(\x)=\sum_{\y\in \N^{\dd}}\trans_t(\y)\vec{\rate}_{\y\x}(t)\\
&=\sum_{i=1}^{\dd}\left\{\ru^i_t\trans_t(\x-e_i)-\frac{\rd^i_t}{\Rd^i_t}\den^i(\x+e_i,t)\trans_t(\x+e_i)-\left(\ru^i_t-\frac{\rd^i_t}{\Rd^i_t}\den^i(\x,t)\right)\trans_t(\x)\right\}.
\end{split}
\end{equation}
To establish \eqref{eq:kolmogorov_to_prove} write
\begin{align*}
&\partial_t\trans_t(\x)=\partial_t\sum_{\x_0\in \N^{\dd}} \sum_{\y\in \N^{\dd}}\binn_{\x_0,\Rd_t}(\y)\,\target(\x_0)\,\poss_{\Ru_t}(\x-\y)\\
&=\sum_{\x_0\in \N^{\dd}} \sum_{\y\in \N^{\dd}}\target(\x_0)\left\{\poss_{\Ru_t}(\x-\y)\partial_t\binn_{\x_0,\Rd_t}(\y)+\binn_{\x_0,\Rd_t}(\y)\partial_t\poss_{\Ru_t}(\x-\y)\right\}\\
&\overset{\ref{eq:bin_pde_idn}}{=}\sum_{\x_0\in \N^{\dd}} \sum_{\y\in \N^{\dd}}\target(\x_0)\poss_{\Ru_t}(\x-\y)\sum_{i=1}^{\dd}\frac{\rd_t^i}{\Rd^i_t}\left(\y^i\binn_{\x_0,\Rd_t}(\y)-(\y^i+1)\binn_{\x_0,\Rd_t}(\y+e_i)\right)\\
&\overset{\ref{eq:poss_pde_idn}}{+}\sum_{\x_0\in \N^{\dd}} \sum_{\y\in \N^{\dd}}\target(\x_0)\binn_{\x_0,\Rd_t}(\y)\sum_{i=1}^{\dd}\ru^i_t\,(\poss_{\Ru_t}(\x-\y-e_i)-\poss_{\Ru_t}(\x-\y))\\
&=:B+ N.
\end{align*}
We now analyze the two terms $B$ and $N$. We have
\begin{align*}
B&=\sum_{i=1}^{\dd}\frac{\rd_t^i}{\Rd^i_t}\sum_{\x_0\in \N^{\dd}} \sum_{\y\in \N^{\dd}}\target(\x_0)\poss_{\Ru_t}(\x-\y)\y^i\binn_{\x_0,\Rd_t}(\y)\\
&-\sum_{i=1}^{\dd}\frac{\rd_t^i}{\Rd^i_t}\sum_{\x_0\in \N^{\dd}} \sum_{\y\in \N^{\dd}}\target(\x_0)\poss_{\Ru_t}(\x-\y)(\y^i+1)\binn_{\x_0,\Rd_t}(\y+e_i)\\
&=\trans_t(\x)\sum_{i=1}^{\dd}\frac{\rd_t^i}{\Rd^i_t}\sum_{\y\in \N^{\dd}}p_{\binnrv_t|\pp_t=\x}(\y)\y^i-\sum_{i=1}^{\dd}\frac{\rd_t^i}{\Rd^i_t}\trans_t(\x+e_i)\sum_{\y\in \N^{\dd}}p_{\binnrv_t|\pp_t=\x+e_i}(\y)\y^i\\
&=\sum_{i=1}^{\dd}\frac{\rd_t^i}{\Rd^i_t}\,[\trans_t(\x)\den^i(\x,t)-\trans_t(\x+e_i)\den^i(\x+e_i,t)],
\end{align*}
while
\begin{align*}
N&=\sum_{\x_0\in \N^{\dd}} \sum_{\y\in \N^{\dd}}\target(\x_0)\binn_{\x_0,\Rd_t}(\y)\sum_{i=1}^{\dd}\ru^i_t\,(\poss_{\Ru_t}(\x-\y-e_i)-\poss_{\Ru_t}(\x-\y))=\sum_{i=1}^{\dd}\ru^i_t\,(\trans_t(\x-e_i)-\trans_t(\x)).
\end{align*}
It follows that 
\begin{align*}
&\partial_t\trans_t(\x)=\sum_{i=1}^{\dd}\left\{\frac{\rd_t^i}{\Rd^i_t}[\trans_t(\x)\den^i(\x,t)-\trans_t(\x+e_i)\den^i(\x+e_i,t)]+\ru^i_t\,[\trans_t(\x-e_i)-\trans_t(\x)]\right\},
\end{align*}
which is exactly \cref{eq:kolmogorov_to_prove}. 
\end{proof}
We can now derive the rates of the time-reversal of $(\vec{\ppa_t})_{t\in [0,\TT]}$.
\begin{corollary}
\label{cor:reverse_time_poss_bin}
With $s=\TT-t$, the time reversal $(\cev{\ppa_s})_{s\in [0,\TT]}$ of $(\vec{\ppa_t})_{t\in [0,\TT]}$ has, for $0<s<\TT$, rate matrices
\begin{equation}
\label{eq:reverse_mat_direct_binposs}
\cev{\rate}_{\x\y}(s)=\begin{cases}
\frac{\ru^i_{\TT-s}}{\Ru^i_{\TT-s}}(\x^{\idx}-\den^i(\x,\TT-s)) \qquad &\textnormal{if}\qquad \y=\x-e_i,\qquad i\in [\dd]\\
-\frac{\rd_{\TT-s}^i\Ru^i_{\TT-s}}{\Rd^i_{\TT-s}}\frac{\den^i(\x+e_i,\TT-s)}{\x^{\idx}+1-\den^i(\x+e_i,\TT-s)}\qquad &\textnormal{if}\qquad \y=\x+e_i,\qquad i\in [\dd]\\
0 \qquad & \textnormal{if} \qquad \y\notin\{\x, \x\pm e_1,\ldots, \x\pm e_{\dd}\}\\
 -\sum_{\y\in \{\x\pm e_1,\ldots,\x\pm e_{\dd}\}}\cev{\rate}_{\x\y}(s)\qquad &\textnormal{if}\qquad \y=\x
\end{cases}.
\end{equation}
\end{corollary}
\begin{proof}
By the time-reversal formula
\[
\cev{\rate}_{\x\y}(s)=\frac{\trans_{\TT-s}(\y)}{\trans_{\TT-s}(\x)}\,\vec{\rate}_{\y\x}(\TT-s),
\]
and \cref{eq:tweedie}, we have
\begin{align*}
\cev{\rate}_{\x,\x-e_i}(s)=\frac{\trans_{\TT-s}(\x-e_i)} {\trans_{\TT-s}(\x)} \vec{\rate}_{\x-e_i,\x}(\TT-s)=\frac{\x^{\idx}-\den^i(\x,\TT-s)}{\Ru^i_{\TT-s}}\ru^i_{\TT-s},
\end{align*}
while
\begin{align*}
\cev{\rate}_{\x,\x+e_i}(s)&=\frac{\trans_{\TT-s}(\x+e_i)} {\trans_{\TT-s}(\x)} \vec{\rate}_{\x+e_i,\x}(\TT-s)\\
&=-\frac{\Ru^i_{\TT-s}}{\x^{\idx}+1-\den^i(\x+e_i,\TT-s)} \frac{\rd^i_{\TT-s}}{\Rd^i_{\TT-s}}\den^i(\x+e_i,\TT-s)\\
&=-\rd^i_{\TT-s}\frac{\Ru^i_{\TT-s}}{\Rd^i_{\TT-s}}\frac{\den^i(\x+e_i,\TT-s)}{\x^{\idx}+1-\den^i(\x+e_i,\TT-s)}.
\end{align*}
\end{proof}

\subsection{Proof of \cref{thm:main_intro}}
\label{app:proof_main}
The time-reversed chain in \cref{cor:reverse_time_poss_bin}, indexed by sampling time $s$, starts from $\trans_{\TT}=\pos(\Ru_{\TT})$ and has marginals $\trans_{\TT-s}$. We use \cref{eq:latents_data_latents}-\cref{eq:latents_data_latents2}, together with \cref{eq:den-dendd2} and $\Rdi_t=1-\Rd_t$, to express its rates in terms of the \texttt{latent-up} and \texttt{latent-down} denoisers, yielding \cref{eq:chain_intro}.

\section{The Poisson-Poisson process}
\label{app:poisson_poisson_process}
The main result of this section is the Poisson-Poisson part of \cref{thm:main}:
\begin{theorem}[The Poisson-Poisson process]
\label{thm:main_p_intro}
Under the schedule assumptions of \cref{thm:main}, set the noise schedules $\Rdi_t, \Ru_t\in \R_{\ge 0}^{\dd}$ to satisfy 
\begin{equation}
\label{eq:posspossrates}
\Rdi^i_0=0, \quad\Ru^i_0=0,  \qquad \ru^i_t:=\frac{\diff}{\diff t}\Ru^i_t,\quad \rdi^i_t:=\frac{\diff}{\diff t}\Rdi^i_t, \qquad i\in [\dd].
\end{equation}
Let $(\transp_t)_{t\in [0,\TT]}$ be the marginals of $(\ppp_t)_{t\in [0,\TT]}$,
\begin{equation}
\label{eq:perturb_pp_intro_thm}  
\ppp_t:=\data-\lrv_t+\possrv_t,\qquad\textnormal{with} \qquad \lrv_t\sim \pos(\Rdi_t),\quad \possrv_t\sim \pos(\Ru_t),
\end{equation} 
where $(\data,\lrv_t,\possrv_t)$ are independent, and define the latent denoisers 
\begin{equation}
\label{eq:denoisers_pos_pos_thm}
\denbin(\x,t):=\EE[\lrv_t|\ppp_t=\x],\qquad \denpos(\x,t):=\EE[\possrv_t|\ppp_t=\x].
\end{equation}
Write $s=\TT-t$ for reverse time. The time reversal $(\pppa_s)_{s\in [0,\TT]}$ of the forward CTMC in \cref{prop:aux_markov_p} has initial distribution $\pppa_0\sim\transp_{\TT}$ and, for $0<s<\TT$, rate matrices
\begin{equation}
\label{eq:reverse_mat_direct_thm}
\rate_{\x\y}(s)=\begin{cases}
\frac{\ru_{\TT-s}^i}{\Ru_{\TT-s}^i}\,\denpos^i(\x,\TT-s) \qquad &\textnormal{if}\qquad \y=\x-e_i,\qquad i\in [\dd]\\
\frac{\rdi_{\TT-s}^i}{\Rdi^i_{\TT-s}}\,\denbin^i(\x,\TT-s)\qquad &\textnormal{if}\qquad \y=\x+e_i,\qquad i\in [\dd]\\
0 \qquad & \textnormal{if} \qquad \y\notin\{\x, \x\pm e_1,\ldots, \x\pm e_{\dd}\}\\
 -\sum_{\y\in \{\x\pm e_1,\ldots,\x\pm e_{\dd}\}}\rate_{\x\y}(s)\qquad &\textnormal{if}\qquad \y=\x
\end{cases}.
\end{equation}
In particular,
\begin{equation}
\label{eq:markov_proj_p_intro}
\pppa_s\sim\transp_{\TT-s}\qquad\textnormal{for all} \qquad s\in [0,\TT].
\end{equation}
\end{theorem}
To avoid confusion with the Binomial-Poisson process, in this section we change the notation according to
\begin{equation}
\label{eq:change_nota_pospos}
\lrv_t\mapsto \drv_t,\qquad \possrv_t\mapsto \urv_t,\qquad
\Rdi_t\mapsto \Rd_t,\qquad \rdi_t\mapsto \rd_t,
\end{equation}
so we focus on the perturbations of $\data\sim\target$, for a given $t\in [0,\TT]$, by
\begin{equation}
\label{eq:perturb_pp_app}  
\ppp_t:=\data-\drv_t+\urv_t,\qquad\textnormal{with} \qquad \drv_t\sim \pos(\Rd_t),\quad \urv_t\sim \pos(\Ru_t),
\end{equation} 
where $(\data,\drv_t,\urv_t)$ are independent.
We begin by recording some of the conditional distributions of $(\data,\drv_t,\urv_t)$.
\begin{lemma}
Let $\transp_t$ be the distribution of $\ppp_t$ given in \eqref{eq:perturb_pp_app}. 
\label{lem:conditional_dist_p}
For $\x\in \Z^{\dd}$,
\begin{equation}
\label{eq:marg_exp_poss_poss}
\begin{split}
\transp_t(\x)&=\sum_{\x_0\in \Z^{\dd}}\sum_{\y\in \Z^{\dd}}\target(\x_0)\,\poss_{\Rd_t}(\y)\,\poss_{\Ru_t}(\x+\y-\x_0)\\
&=\sum_{\x_0\in \Z^{\dd}}\sum_{\y\in \Z^{\dd}}\target(\x_0)\,\poss_{\Ru_t}(\y)\,\poss_{\Rd_t}(\y+\x_0-\x).
\end{split}
\end{equation}
Fix $\x\in \Z^{\dd}$, then, for $\y\in  \Z^{\dd}$, 
\begin{align}
&\transp_{\drv_t|\ppp_t=\x}(\y)=\frac{\poss_{\Rd_t}(\y)}{\transp_t(\x)}\sum_{\x_0\in\Z^{\dd}}\target(\x_0)\,\poss_{\Ru_t}(\x+\y-\x_0),\label{eq:Dt_given_yt}\\
&\transp_{\urv_t|\ppp_t=\x}(\y)=\frac{\poss_{\Ru_t}(\y)}{\transp_t(\x)}\sum_{\x_0\in\Z^{\dd}}\target(\x_0)\,\poss_{\Rd_t}(\x_0-\x+\y). \label{eq:Ut_given_yt}
\end{align}

Fix $\x,\z\in \Z^{\dd}$, then, for $\y\in  \Z^{\dd}$, 
\begin{align}
&\transp_{\drv_t|\ppp_t=\x,\data =\z}(\y)=\frac{\poss_{\Rd_t}(\y)\poss_{\Ru_t}(\x+\y-\z)}{\sum_{\tilde{\y}\in \Z^{\dd}}\poss_{\Rd_t}(\tilde{\y})\poss_{\Ru_t}(\x+\tilde{\y}-\z)},\label{eq:Dt_given_ytdata}\\
&\transp_{\urv_t|\ppp_t=\x,\data =\z}(\y)=\frac{\poss_{\Ru_t}(\y)\poss_{\Rd_t}(\y+\z-\x)}{\sum_{\tilde{\y}\in \Z^{\dd}}\poss_{\Ru_t}(\tilde{\y})\poss_{\Rd_t}(\tilde{\y}+\z-\x)}.
\label{eq:Ut_given_ytdata}
\end{align}

Fix $\y\in \Z^{\dd}$, then, for $\x\in  \Z^{\dd}$, 
\begin{equation}
\label{eq:data_given_Yt}
\transp_{\data|\ppp_t=\y}(\x)=\frac{\target(\x)}{\transp_t(\y)}(\poss_{\Rd_t}(\x-\cdot)\ast \poss_{\Ru_t})(\y).
\end{equation}
\end{lemma}
\begin{proof}
\cref{eq:marg_exp_poss_poss} holds by the definition \eqref{eq:perturb_pp_app}. For \cref{eq:Dt_given_yt} and \cref{eq:Ut_given_yt} we use  Bayes' rule to get
\begin{align*}
\transp_{\drv_t|\ppp_t=\x}(\y)&=\frac{\transp_{\drv_t,\ppp_t}(\y,\x)}{\transp_{\ppp_t}(\x)}=\frac{\transp_{\drv_t,\data+\urv_t}(\y,\x+\y)}{\transp_t(\x)}=\frac{\transp_{\drv_t}(\y)\transp_{\data+\urv_t}(\x+\y)}{\transp_t(\x)}\\
&=\frac{\poss_{\Rd_t}(\y)}{\transp_t(\x)}\sum_{\x_0\in\Z^{\dd}}\target(\x_0)\,\poss_{\Ru_t}(\x+\y-\x_0),
\end{align*}
and
\begin{align*}
\transp_{\urv_t|\ppp_t=\x}(\y)&=\frac{\transp_{\urv_t,\ppp_t}(\y,\x)}{\transp_{\ppp_t}(\x)}=\frac{\transp_{\urv_t,\data-\drv_t}(\y,\x-\y)}{\transp_t(\x)}=\frac{\transp_{\urv_t}(\y)\transp_{\data-\drv_t}(\x-\y)}{\transp_t(\x)}\\
&=\frac{\poss_{\Ru_t}(\y)}{\transp_t(\x)}\sum_{\x_0\in\Z^{\dd}}\target(\x_0)\,\poss_{\Rd_t}(\x_0-\x+\y).
\end{align*}

For \cref{eq:Dt_given_ytdata} and \cref{eq:Ut_given_ytdata} we use the independence of $(\data,\drv_t,\urv_t)$ and Bayes' rule to get
\begin{align*}
&\transp_{\drv_t|\ppp_t=\x,\data =\z}(\y)=\frac{\transp_{\drv_t,\ppp_t,\data }(\y,\x,\z)}{\transp_{\ppp_t,\data }(\x,\z)}=\frac{\transp_{\drv_t,\urv_t,\data }(\y,\x+\y-\z,\z)}{\transp_{\ppp_t,\data }(\x,\z)}\\
&=\frac{\transp_{\drv_t}(\y)\transp_{\urv_t}(\x+\y-\z)\target(\z)}{\transp_{\ppp_t|\data=\z}(\x)\target(\z)}=\frac{\poss_{\Rd_t}(\y)\poss_{\Ru_t}(\x+\y-\z)}{\transp_{\ppp_t|\data=\z}(\x)}=\frac{\poss_{\Rd_t}(\y)\poss_{\Ru_t}(\x+\y-\z)}{\sum_{\tilde{\y}\in \Z^{\dd}}\poss_{\Rd_t}(\tilde{\y})\poss_{\Ru_t}(\x+\tilde{\y}-\z)},
\end{align*}
and similarly, 
\begin{align*}
&\transp_{\urv_t|\ppp_t=\x,\data =\z}(\y)=\frac{\transp_{\urv_t,\ppp_t,\data }(\y,\x,\z)}{\transp_{\ppp_t,\data }(\x,\z)}=\frac{\transp_{\urv_t,\drv_t,\data }(\y,\y+\z-\x,\z)}{\transp_{\ppp_t,\data }(\x,\z)}\\
&=\frac{\transp_{\urv_t}(\y)\transp_{\drv_t}(\y+\z-\x)\target(\z)}{\transp_{\ppp_t|\data=\z}(\x)\target(\z)}=\frac{\poss_{\Ru_t}(\y)\poss_{\Rd_t}(\y+\z-\x)}{\transp_{\ppp_t|\data=\z}(\x)}=\frac{\poss_{\Ru_t}(\y)\poss_{\Rd_t}(\y+\z-\x)}{\sum_{\tilde{\y}\in \Z^{\dd}}\poss_{\Ru_t}(\tilde{\y})\poss_{\Rd_t}(\tilde{\y}+\z-\x)}.
\end{align*}
\cref{eq:data_given_Yt} also follows from  Bayes' rule.
 \end{proof}
\subsection{Discrete Tweedie's formula}
\label{app:tweedie_p}
In this section we derive a discrete Tweedie's formula for the Poisson-Poisson perturbation \eqref{eq:perturb_pp_app} . This formula is crucial in proving \cref{thm:main_p_intro}. 
\begin{lemma}[Discrete Tweedie's formula]
\label{lem:tweedie_poss_poss}
Fix $t\in (0,\TT)$, and let
\begin{equation}
\label{eq:up_down_p_pertb_lem}
\ppp_t:=\data-\drv_t+\urv_t,\qquad\textnormal{with} \qquad \drv_t\sim \pos(\Rd_t),\quad \urv_t\sim \pos(\Ru_t),
\end{equation}
where $(\data,\drv_t,\urv_t)$ are independent.
 Denote by $\transp_t$ the distribution of $\ppp_t$, i.e., $\ppp_t\sim \transp_t$, and define 
\begin{equation}
\label{eq:denoiser_D_app}
\denbin(\x,t):=\EE[\drv_t|\ppp_t=\x],\qquad \denpos(\x,t):=\EE[\urv_t|\ppp_t=\x].
\end{equation}
Then,
\begin{equation}
\label{eq:tweedie_poss_poss_down}
\frac{\transp_t(\x+e_i)}{\transp_t(\x)}=\frac{1}{\Rd^i_t} \,\denbin^i(\x,t),\qquad \x\in\Z^{\dd},\quad t\in (0,\TT),\quad i\in [\dd],
\end{equation}
and 
\begin{equation}
\label{eq:tweedie_poss_poss_up}
\frac{\transp_t(\x-e_i)}{\transp_t(\x)}=\frac{\denpos^i(\x,t)}{\Ru_t^i}
,\qquad \x\in\Z^{\dd},\quad t\in (0,\TT),\quad i\in [\dd].
\end{equation}
\end{lemma}
\begin{proof}
\begin{align*}
&\denbin^i(\x,t)=\EE[\drv_t^i|\ppp_t=\x]=\sum_{\y\in\Z^{\dd}}\y^i\transp_{\drv_t|\ppp_t=\x}(\y)\\
&\overset{\ref{eq:Dt_given_yt}}{=}\frac{1}{\transp_t(\x)}\sum_{\y\in\Z^{\dd}}\y^i 
\poss_{\Rd_t}(\y)\sum_{\x_0\in\Z^{\dd}}\target(\x_0)\,\poss_{\Ru_t}(\x+\y-\x_0)\\
&\overset{\ref{eq:poss_idn}}{=}\frac{1}{\transp_t(\x)}\sum_{\y\in\Z^{\dd}}\Rd^i_t\,\poss_{\Rd_t}(\y-e_i)\sum_{\x_0\in\Z^{\dd}}\target(\x_0)\,\poss_{\Ru_t}(\x+\y-\x_0)\\
&=\frac{\Rd^i_t}{\transp_t(\x)}\sum_{\y\in\Z^{\dd}}\sum_{\x_0\in\Z^{\dd}}\target(\x_0)\,\poss_{\Rd_t}(\y-e_i)\,\poss_{\Ru_t}(\x+(\y-e_i)-(\x_0-e_i))\\
&\overset{\textnormal{re-index}}{=}\frac{\Rd^i_t}{\transp_t(\x)}\sum_{\y\in\Z^{\dd}}\sum_{\x_0\in\Z^{\dd}}\target(\x_0)\,\poss_{\Rd_t}(\y)\,\poss_{\Ru_t}((\x+e_i)+\y-\x_0)\\
&\overset{\ref{eq:marg_exp_poss_poss}}{=}\Rd^i_t\,\frac{\transp_t(\x+e_i)}{\transp_t(\x)},
\end{align*}
and similarly, 
\begin{align*}
&\denpos^i(\x,t)=\EE[\urv_t^i|\ppp_t=\x]=\sum_{\y\in\Z^{\dd}}\y^i\transp_{\urv_t|\ppp_t=\x}(\y)\\
&\overset{\ref{eq:Ut_given_yt}}{=}\frac{1}{\transp_t(\x)}\sum_{\y\in\Z^{\dd}}\y^i 
\poss_{\Ru_t}(\y)\sum_{\x_0\in\Z^{\dd}}\target(\x_0)\,\poss_{\Rd_t}(\x_0-\x+\y)\\
&\overset{\ref{eq:poss_idn}}{=}\frac{1}{\transp_t(\x)}\sum_{\y\in\Z^{\dd}}\Ru_t^i\, 
\poss_{\Ru_t}(\y-e_i)\sum_{\x_0\in\Z^{\dd}}\target(\x_0)\,\poss_{\Rd_t}(\x_0-\x+\y)\\
&=\frac{\Ru_t^i}{\transp_t(\x)}\sum_{\y\in\Z^{\dd}}
\sum_{\x_0\in\Z^{\dd}}\target(\x_0)\,\poss_{\Ru_t}(\y-e_i)\,\poss_{\Rd_t}(\x_0-(\x-e_i)+(\y-e_i))\\
&\overset{\textnormal{re-index}}{=}\frac{\Ru_t^i}{\transp_t(\x)}\sum_{\y\in\Z^{\dd}}
\sum_{\x_0\in\Z^{\dd}}\target(\x_0)\,\poss_{\Ru_t}(\y)\,\poss_{\Rd_t}(\x_0-(\x-e_i)+\y)\\
&\overset{\ref{eq:marg_exp_poss_poss}}{=}\Ru_t^i\,\frac{\transp_t(\x-e_i)}{\transp_t(\x)}.
\end{align*}
\end{proof}
We can now show that the \texttt{latent-up} and \texttt{latent-down} denoisers can be expressed in terms of each other, in a non-linear and non-local way. Consequently, the \texttt{data} denoiser can be expressed in terms of a single latent denoiser,  in a non-linear and non-local way.
\begin{corollary}[Relations between \texttt{latent-up}, \texttt{latent-down}, and \texttt{data} denoisers]
\label{cor:latentupdown_possposs}
Fix $\x\in\Z^{\dd},t\in (0,\TT)$ and $i\in [\dd]$. Then,
\begin{equation}
\label{eq:tweedie_poss_poss_bin_up}
\denpos^i(\x,t)=\frac{\Ru_t^i\Rd^i_t}{\denbin^i(\x-e_i,t)},
\end{equation} 
and 
\begin{align}
\dend^i(\x,t)&=\x^i-\denpos^i(\x,t)+\frac{\Rd_t^i\Ru_t^i}{\denpos^i(\x+e_i,t)}\label{eq:den-dendd2_p}\\
&=\x^i+\denbin^i(\x,t)-\frac{\Ru_t^i\Rd^i_t}{\denbin^i(\x-e_i,t)}.\label{eq:den-dendd3_p}
\end{align}
\end{corollary}
\begin{proof}
For \cref{eq:tweedie_poss_poss_bin_up} we use
\begin{align*}
\denpos^i(\x,t)\overset{\ref{eq:tweedie_poss_poss_up}}{=}\Ru_t^i\frac{\transp_t(\x-e_i)}{\transp_t(\x)}\overset{\ref{eq:tweedie_poss_poss_down}}{=}\frac{\Ru_t^i\Rd^i_t}{\denbin^i(\x-e_i,t)}.
\end{align*}
\cref{eq:den-dendd2_p} and \cref{eq:den-dendd3_p} then follow from \cref{eq:tweedie_poss_poss_bin_up} and the relation
\[
\dend(\x,t)=\x+\denbin(\x,t)-\denpos(\x,t).
\]
\end{proof}

We conclude by expressing the latent denoisers in terms of the marginal posteriors of the data given its noisy observations. 
\begin{lemma}
\label{lem:cor_cond_exp_pospos_posterior}
Fix $i\in [\dd]$ and $\x\in \Z^{\dd}$. Then,
\begin{equation}
\label{eq:cor_cond_exp_pospos_posterior_up}
\denpos^i(\x,t)=\sum_{\z^i\in\Z}\transp_{\data^i|\ppp_t=\x}(\z^i)\sum_{\y^i\in \Z}\y^i\,\frac{\poss_{\Ru_t^i}(\y^i)\poss_{\Rd_t^i}(\y^i+\z^i-\x^i)}{\sum_{w\in \Z}\poss_{\Ru_t^i}(w)\poss_{\Rd_t^i}(w+\z^i-\x^i)},
\end{equation}
and
\begin{equation}
\label{eq:cor_cond_exp_pospos_posterior_down}
\denbin^i(\x,t)=\sum_{\z^i\in\Z}\transp_{\data^i|\ppp_t=\x}(\z^i)\sum_{\y^i\in \Z}\y^i\,\frac{\poss_{\Rd_t^i}(\y^i)\poss_{\Ru_t^i}(\x^i+\y^i-\z^i)}{\sum_{w\in \Z}\poss_{\Rd_t^i}(w)\poss_{\Ru_t^i}(\x^i+w-\z^i)}.
\end{equation}
\end{lemma}
\begin{proof}
We first note that \cref{eq:Ut_given_ytdata} shows that $\transp_{\urv_t|\ppp_t=\x,\data=\z}$ is a product distribution with marginals
\begin{equation}
\label{Ut_given_Ytxdata_i}
\transp_{\urv_t^i|\ppp_t=\x,\data=\z}(\y^i)=\frac{\poss_{\Ru_t^i}(\y^i)\poss_{\Rd_t^i}(\y^i+\z^i-\x^i)}{\sum_{w\in \Z}\poss_{\Ru_t^i}(w)\poss_{\Rd_t^i}(w+\z^i-\x^i)}.
\end{equation}
By marginalization it follows that 
\begin{align*}
&\transp_{\urv_t^i|\ppp_t=\x,\data^i=\z^i}(\y^i)=\sum_{\z^{-i}\in\Z^{\dd-1}}\transp_{\urv_t^i|\ppp_t=\x,\data^i=\z^i, \data^{-i}=\z^{-i}}(\y^i)\transp_{\data^{-i}|\ppp_t=\x,\data^i=\z^i}(\z^{-i})\\
&=\sum_{\z^{-i}\in\Z^{\dd-1}}\transp_{\urv_t^i|\ppp_t=\x,\data=\z}(\y^i)\transp_{\data^{-i}|\ppp_t=\x,\data^i=\z^i}(\z^{-i})\\
&=\transp_{\urv_t^i|\ppp_t=\x,\data=\z}(\y^i)\underbrace{\sum_{\z^{-i}\in\Z^{\dd-1}}\transp_{\data^{-i}|\ppp_t=\x,\data^i=\z^i}(\z^{-i})}_{=1}=\transp_{\urv_t^i|\ppp_t=\x,\data=\z}(\y^i).
\end{align*}
Combining with \cref{Ut_given_Ytxdata_i} we conclude
\begin{equation}
\label{Ut_given_Ytxdata_i_reduce_marg}    
\transp_{\urv_t^i|\ppp_t=\x,\data^i=\z^i}(\y^i)=\frac{\poss_{\Ru_t^i}(\y^i)\poss_{\Rd_t^i}(\y^i+\z^i-\x^i)}{\sum_{w\in \Z}\poss_{\Ru_t^i}(w)\poss_{\Rd_t^i}(w+\z^i-\x^i)}.
\end{equation}
Hence, by Bayes' rule,
\begin{align*}
&\denpos^i(\x,t)=\EE[\urv_t^i|\ppp_t=\x]=\sum_{\y^i\in \Z}\y^i\,\transp_{\urv_t^i|\ppp_t=\x}(\y^i)\\
&=\sum_{\y^i\in \Z}\y^i\,\frac{\transp_{\urv_t^i,\ppp_t}(\y^i, \x)}{\transp_{\ppp_t}(\x)}=\sum_{\y^i\in \Z}\y^i\frac{1}{\transp_{\ppp_t}(\x)}\sum_{\z^i\in\Z}\transp_{\urv_t^i,\ppp_t,\data^i}(\y^i, \x,\z^i)\\
&=\sum_{\y^i\in \Z}\y^i\sum_{\z^i\in\Z}\transp_{\urv_t^i|\ppp_t=\x,\data^i=\z^i}(\y^i)\frac{\transp_{\ppp_t,\data^i}(\x,\z^i)}{\transp_{\ppp_t}(\x)}\\
&=\sum_{\y^i\in \Z}\y^i\sum_{\z^i\in\Z}\transp_{\urv_t^i|\ppp_t=\x,\data^i=\z^i}(\y^i)\transp_{\data^i|\ppp_t=\x}(\z^i)\\
&\overset{\ref{Ut_given_Ytxdata_i_reduce_marg}}{=}\sum_{\z^i\in\Z}\transp_{\data^i|\ppp_t=\x}(\z^i)\sum_{\y^i\in \Z}\y^i\,\frac{\poss_{\Ru_t^i}(\y^i)\poss_{\Rd_t^i}(\y^i+\z^i-\x^i)}{\sum_{w\in \Z}\poss_{\Ru_t^i}(w)\poss_{\Rd_t^i}(w+\z^i-\x^i)},
\end{align*}
and similarly, using \cref{eq:Dt_given_ytdata},
\begin{align*}
\denbin^i(\x,t)=\sum_{\z^i\in\Z}\transp_{\data^i|\ppp_t=\x}(\z^i)\sum_{\y^i\in \Z}\y^i\,\frac{\poss_{\Rd_t^i}(\y^i)\poss_{\Ru_t^i}(\x^i+\y^i-\z^i)}{\sum_{w\in \Z}\poss_{\Rd_t^i}(w)\poss_{\Ru_t^i}(\x^i+w-\z^i)}.
\end{align*}
\end{proof}

\subsection{Forward and backward processes}
\label{app:for_back_p}
In this section we will show how to use the latent denoisers in order to define forward and  backward processes with the marginals of the Poisson-Poisson perturbation \eqref{eq:perturb_pp_app}. Consequently, we prove \cref{thm:main_p_intro}. 
For $i\in [\dd]$ and $t\in [0,\TT]$ set 
\begin{equation}
\label{eq:lam_Lam_bin} 
\Ru_0^i=0,\quad \Rd_0^i=0,\qquad \ru^i_t:=\frac{\diff }{\diff t}\Ru^i_t,\qquad \rd^i_t:=\frac{\diff }{\diff t}\Rd^i_t. 
\end{equation}
We begin by deriving the evolution equation of the marginals of the Poisson-Poisson perturbation \eqref{eq:perturb_pp_app}.
\begin{lemma}
\label{lem:marg_pos_poss}
Let $\ppp_t\sim\transp_t$ where $\ppp_t$ is as in \cref{eq:perturb_pp_app}. Then, for $\x\in \Z^{\dd}$ and $t\in [0,\TT]$,
\begin{equation}
\label{eq:marg_pos_poss}    
\partial_t\transp_t(\x)=\sum_{i=1}^{\dd}\rd_t^i\,\transp_t(\x+e_i)+\sum_{i=1}^{\dd}\ru_t^i\,\transp_t(\x-e_i)-\transp_t(\x)\sum_{i=1}^{\dd}(\rd_t^i+\ru_t^i).
\end{equation}
\end{lemma}

\begin{proof}
\begin{align*}
&\partial_t\transp_t(\x)\overset{\ref{eq:marg_exp_poss_poss}}{=}\partial_t\sum_{\x_0\in \Z^{\dd}}\sum_{\y\in \Z^{\dd}}\target(\x_0)\,\poss_{\Rd_t}(\y)\,\poss_{\Ru_t}(\x+\y-\x_0)\\
&=\sum_{\x_0\in \Z^{\dd}}\sum_{\y\in \Z^{\dd}}\target(\x_0)\,\poss_{\Ru_t}(\x+\y-\x_0)\partial_t\poss_{\Rd_t}(\y)+\sum_{\x_0\in \Z^{\dd}}\sum_{\y\in \Z^{\dd}}\target(\x_0)\,\poss_{\Rd_t}(\y)\partial_t\poss_{\Ru_t}(\x+\y-\x_0)\\
&\overset{\ref{eq:poss_pde_idn}}{=}\sum_{\x_0\in \Z^{\dd}}\sum_{\y\in \Z^{\dd}}\target(\x_0)\,\poss_{\Ru_t}(\x+\y-\x_0)\left[\sum_{i=1}^{\dd}\rd^i_t\,(\poss_{\Rd_t}(\y-e_i)-\poss_{\Rd_t}(\y))\right]\\
&+\sum_{\x_0\in \Z^{\dd}}\sum_{\y\in \Z^{\dd}}\target(\x_0)\,\poss_{\Rd_t}(\y)\left[\sum_{i=1}^{\dd}\ru^i_t\,(\poss_{\Ru_t}(\x+\y-\x_0-e_i)-\poss_{\Ru_t}(\x+\y-\x_0))\right]\\
&=\sum_{i=1}^{\dd}\rd_t^i\left(\sum_{\x_0\in \Z^{\dd}}\sum_{\y\in \Z^{\dd}}\target(\x_0)\,\poss_{\Ru_t}((\x+e_i)+(\y-e_i)-\x_0)\,\poss_{\Rd_t}(\y-e_i)\right)-\transp_t(\x)\sum_{i=1}^{\dd}\rd_t^i\\
&+\sum_{i=1}^{\dd}\ru_t^i\left(\sum_{\x_0\in \Z^{\dd}}\sum_{\y\in \Z^{\dd}}\target(\x_0)\,\poss_{\Rd_t}(\y)\,\poss_{\Ru_t}((\x-e_i)+\y-\x_0)\right)-\transp_t(\x)\sum_{i=1}^{\dd}\ru_t^i\\
&=\sum_{i=1}^{\dd}\rd_t^i\,\transp_t(\x+e_i)+\sum_{i=1}^{\dd}\ru_t^i\,\transp_t(\x-e_i)-\transp_t(\x)\sum_{i=1}^{\dd}(\rd_t^i+\ru_t^i).
\end{align*}
\end{proof}
Next we will show how to use the denoisers \eqref{eq:denoisers_pos_pos_thm} to define a CTMC whose marginals are the same as the Poisson-Poisson perturbation \eqref{eq:perturb_pp_app}.

\begin{proposition}
\label{prop:aux_markov_p}
Let $(\vec{\pppa_t})_{t\in [0,\TT]}$ be the CTMC with rate matrices
\begin{equation}
\label{eq:rate_axu_possposs}
\vec{\rate}_{\x\y}(t)=
\begin{cases}
\ru^i_t \qquad &\textnormal{if}\qquad \y=\x+e_i,\qquad i\in [\dd]\\
\rd^i_t\qquad &\textnormal{if}\qquad \y=\x-e_i,\qquad i\in [\dd]\\
0 \qquad & \textnormal{if} \qquad \y\notin\{\x, \x\pm e_1,\ldots, \x\pm e_{\dd}\}\\
 -\sum_{\y\in \{\x\pm e_1,\ldots,\x\pm e_{\dd}\}}\vec{\rate}_{\x\y}(t)\qquad &\textnormal{if}\qquad \y=\x
\end{cases},
\end{equation}
and $\vec{\pppa_0}\sim \target$. Then,
\[
\vec{\pppa_t}\sim \transp_t\qquad\textnormal{for all} \qquad t\in [0,\TT].
\]
\end{proposition}
\begin{proof}
It suffices to show that $(\transp_t)_{t\in [0,\TT]}$ satisfies the Kolmogorov equation
\begin{equation}
\label{eq:kolmogorov_to_prove1}
\begin{split}
&\partial_t\transp_t(\x)=\sum_{\y\in \Z^{\dd}}\transp_t(\y)\vec{\rate}_{\y\x}(t).
\end{split}
\end{equation}
This holds true by \cref{eq:marg_pos_poss}.   
\end{proof}
We can now derive the rates of the time-reversal of $(\vec{\pppa_t})_{t\in [0,\TT]}$.
\begin{corollary}
\label{cor:reverse_time_poss_pos}
With $s=\TT-t$, the time reversal $(\cev{\pppa_s})_{s\in [0,\TT]}$ of $(\vec{\pppa_t})_{t\in [0,\TT]}$ has, for $0<s<\TT$, rate matrices
\begin{equation}
\label{eq:reverse_mat_direct_possposs}
\cev{\rate}_{\x\y}(s)=\begin{cases}
\frac{\ru_{\TT-s}^i}{\Ru_{\TT-s}^i}\denpos^i(\x,\TT-s) \qquad &\textnormal{if}\qquad \y=\x-e_i,\qquad i\in [\dd]\\
\frac{\rd_{\TT-s}^i}{\Rd^i_{\TT-s}} \,\denbin^i(\x,\TT-s)\qquad &\textnormal{if}\qquad \y=\x+e_i,\qquad i\in [\dd]\\
0 \qquad & \textnormal{if} \qquad \y\notin\{\x, \x\pm e_1,\ldots, \x\pm e_{\dd}\}\\
 -\sum_{\y\in \{\x\pm e_1,\ldots,\x\pm e_{\dd}\}}\cev{\rate}_{\x\y}(s)\qquad &\textnormal{if}\qquad \y=\x
\end{cases}.
\end{equation}
\end{corollary}
\begin{proof}
By the time-reversal formula
\[
\cev{\rate}_{\x\y}(s)=\frac{\transp_{\TT-s}(\y)}{\transp_{\TT-s}(\x)}\,\vec{\rate}_{\y\x}(\TT-s),
\]
so
\begin{align*}
\cev{\rate}_{\x,\x-e_i}(s)=\frac{\transp_{\TT-s}(\x-e_i)} {\transp_{\TT-s}(\x)} \vec{\rate}_{\x-e_i,\x}(\TT-s)\overset{\ref{eq:tweedie_poss_poss_up}}{=}\frac{\ru_{\TT-s}^i}{\Ru_{\TT-s}^i}\denpos^i(\x,\TT-s),
\end{align*}
and
\begin{align*}
\cev{\rate}_{\x,\x+e_i}(s)&=\frac{\transp_{\TT-s}(\x+e_i)} {\transp_{\TT-s}(\x)} \vec{\rate}_{\x+e_i,\x}(\TT-s)\overset{\ref{eq:tweedie_poss_poss_down}}{=}\frac{\rd_{\TT-s}^i}{\Rd^i_{\TT-s}} \,\denbin^i(\x,\TT-s).
\end{align*}
\end{proof}

\subsection{Proof of Theorem~\ref{thm:main_p_intro}}
\label{app:proof_main_p_intro}
The time-reversed chain in \cref{cor:reverse_time_poss_pos}, indexed by sampling time $s$, starts from $\transp_{\TT}$ and has marginals $\transp_{\TT-s}$. Its rates are those in \cref{eq:reverse_mat_direct_thm} under the notation change \cref{eq:change_nota_pospos}.
\section{Jumping only up or only down}
\label{app:only_up_down}
In this section we work out special cases of the Binomial-Poisson process and Poisson-Poisson process where the jumps are uni-directional. The displays below give the off-diagonal rates; the diagonal entries are minus the sum of the off-diagonal entries in the same row. 
\subsection{Jumping up on $\N^{\dd}$ or $\Z^{\dd}$}
\paragraph{Binomial-Poisson.}
\label{app:jump_up}
Take $\Rd^i_t=1$ for all $t\in [0,\TT]$ and $i\in [\dd]$. Then
\[
\pp_t=\data+\possrv_t, \qquad \possrv_t\sim \pos(\Ru_t),
\]
so
\begin{equation}
\label{eq:denoisers_agree}
\den(\x,t)=\EE[\data|\data+\possrv_t=\x]=\dend(\x,t).
\end{equation}
As $\rd^i_t=0$ for all $t\in [0,\TT]$ and $i\in [\dd]$, \cref{eq:chain_intro} reads
\begin{equation}
\label{eq:chain_up}
\rate_{\x\y}(t)=\begin{cases}
\frac{\ru^i_{\TT-t}}{\Ru^i_{\TT-t}}(\x^{\idx}-\dend^i(\x,\TT-t)) \qquad &\textnormal{if}\qquad \y=\x-e_i,\quad i\in [\dd]\\
0\qquad &\textnormal{if}\qquad \y =\x+e_i,\quad i\in [\dd]\\
0 \qquad & \textnormal{if} \qquad \y\notin\{\x, \x\pm e_1,\ldots, \x\pm e_{\dd}\}
\end{cases}.
\end{equation}

\paragraph{Poisson-Poisson.}
Take $\Rd^i_t=0$ for all $t\in [0,\TT]$ and $i\in [\dd]$. Then
\[
\ppp_t=\data+\urv_t,\qquad \urv_t\sim \pos(\Ru_t),
\]
so
\[
\denbin(\x,t)=0,\qquad \denpos(\x,t)=\x-\EE[\data|\data+\urv_t=\x]=\x-\dend(\x,t).
\]
\cref{eq:reverse_mat_direct_thm} is then the same as \cref{eq:chain_up} so the ``only jump up" version of the Binomial-Poisson process and the Poisson-Poisson process is the same.

\subsection{Jumping down on $\N^{\dd}$}
Take $\Ru_t^i=0$ for all $t\in [0,\TT]$ and $i\in [\dd]$. Then 
\[
\pp_t=\binnrv_t,\qquad \binnrv_t\sim\bin(\data,\Rd_t),
\]
so 
\[
\dend(\x,t)=\EE[\data|\binnrv_t=\x],\qquad \den(\x,t)=\x,
\]
and hence \cref{eq:chain_intro} reads 
\begin{equation}
\label{eq:chain_down}
\rate_{\x\y}(t)=\begin{cases}
0 \qquad &\textnormal{if}\qquad \y=\x-e_i,\quad i\in [\dd]\\
-\frac{\rd^i_{\TT-t}}{1-\Rd_{\TT-t}^i}(\dend^i(\x,\TT-t)-\x^i)\qquad &\textnormal{if}\qquad \y =\x+e_i,\quad i\in [\dd]\\
0 \qquad & \textnormal{if} \qquad \y\notin\{\x, \x\pm e_1,\ldots, \x\pm e_{\dd}\}
\end{cases}.
\end{equation}
\cref{eq:chain_down} is exactly the Binomial Flow  of \cite{shenfeld2026binomial}.

\section{EDM scaling}
\label{sec:edm}
In this section we derive centered affine EDM scaling for both the Binomial-Poisson and Poisson-Poisson perturbations. The regression target may itself depend on time: let $A_{\sigma}$ be the target and $Y_{\sigma}$ the observed noisy input, with joint distribution
\[
\pi_{\sigma}(a,y)=\Pr(A_{\sigma}=a,Y_{\sigma}=y).
\]
This includes a fixed data target $\data$ and time-dependent latent targets ($\binnrv_t, \possrv_t, \mathrm L_t$). For the purpose of computing scalar preconditioning coefficients, we assume that the coordinates $i\in \{1,\ldots, \dd\}$ of $A_{\sigma}$ and $Y_{\sigma}$ are i.i.d., with coordinatewise moments
\[
\begin{aligned}
\EE[A_{\sigma}^i]&=\mdata_A(\sigma),&
\var[A_{\sigma}^i]&=\vdata_A(\sigma),\\
\EE[Y_{\sigma}^i]&=\mdata_Y(\sigma),&
\var[Y_{\sigma}^i]&=\vdata_Y(\sigma),&
\cov(A_{\sigma}^i,Y_{\sigma}^i)&=\cdata(\sigma).
\end{aligned}
\]
We parametrize the denoiser as
\begin{equation}
\label{eq:edm_scale_app}
\mathrm m_{\theta}(\y,\sigma)
=\cskip(\sigma)\y+\cbias(\sigma)
+\cout(\sigma)\net_{\theta}(\cin(\sigma)\y+\ssin(\sigma),\sigma),
\end{equation}
where $\net_{\theta}$ is a neural network.

We first normalize the network input. We begin with the variance.
\begin{claim}[$\cin$]
\label{cl:cin}
If 
\begin{equation}
\label{eq:cin}
\cin(\sigma)=\frac{1}{\sqrt{\vdata_Y(\sigma)}},
\end{equation}
then
$\var[\cin(\sigma)Y_{\sigma}^i+\ssin(\sigma)]=1$.
\end{claim}
\begin{proof}
Immediate.
\end{proof}
Next $\ssin$ is chosen to center the input.
\begin{claim}[$\ssin$]
\label{cl:sin}
If 
\begin{equation}
\label{eq:sin}
\ssin(\sigma)=-\frac{\mdata_Y(\sigma)}{\sqrt{\vdata_Y(\sigma)}},
\end{equation}
and $\cin(\sigma)$ is as in \cref{eq:cin}, then
$\EE[\cin(\sigma)Y_{\sigma}^i+\ssin(\sigma)]=0$.
\end{claim}
\begin{proof}
Immediate.
\end{proof}
While the coefficients $\cin$ and $\ssin$ depend only on the statistics of the noisy observation $Y_{\sigma}$, the coefficients $\cbias, \cskip,\cout$ depend also on the target $A_{\sigma}$. We begin by matching the mean of the regression residual  to the mean of the target using $\cbias$. Note that at this point $\cbias$ is expressed in terms of $\cskip$, which will be determined later. 
\begin{claim}[$\cbias$]
\label{cl:cbias}
If  
\begin{equation}
\label{eq:cbias} 
\cbias(\sigma)=\mdata_A(\sigma)-\cskip(\sigma)\mdata_Y(\sigma),
\end{equation}
then $\EE[\cskip(\sigma)Y_{\sigma}^i+\cbias(\sigma)]=\mdata_A(\sigma)$.
\end{claim}
\begin{proof}
Immediate.
\end{proof}
We now determine the $\cskip$ and $\cout$. With \cref{eq:edm_scale_app}, the mean-square error objective reads
\begin{equation}
\label{eq:target_mse}
\begin{split}
&\int\sum_{a,y}\w^2(\sigma)\,|a-\mathrm m_{\theta}(y,\sigma)|^2\,\pi_{\sigma}(a,y)\diff\sigma\\
&=\int\sum_{a,y}\w^2(\sigma)\cout^2(\sigma)\,
|\net_{\theta}(\cin(\sigma)y+\ssin(\sigma),\sigma)
-\net_{\textnormal{target}}(a,y,\sigma)|^2\,\pi_{\sigma}(a,y)\diff\sigma,
\end{split}
\end{equation}
where the sum is over the support of $\pi_{\sigma}$, and
\begin{equation}
 \label{eq:Ftarget_def} 
 \net_{\textnormal{target}}(a,y,\sigma)
:=\frac{a-\cskip(\sigma)y-\cbias(\sigma)}{\cout(\sigma)}.
\end{equation}
By the choice of $\cbias$ according to \cref{eq:cbias}, each coordinate of this target has zero mean. We choose $\cout$ so that each coordinate also has unit variance.
\begin{claim}[Relation between $\cout$ and $\cskip$]
\label{cl:out_skip_rel}
If
\begin{equation}
\label{eq:out_skip_rel}
\cout^2(\sigma)=\vdata_A(\sigma)+\cskip^2(\sigma)\vdata_Y(\sigma)-2\cskip(\sigma)\cdata(\sigma),
\end{equation}
then 
\[
\var[\net_{\textnormal{target}}^i(A_{\sigma},Y_{\sigma},\sigma)]=1.
\]
\end{claim}
\begin{proof}
Subtracting the deterministic bias does not change the variance, so
\[
\var[\net_{\textnormal{target}}^i(A_{\sigma},Y_{\sigma},\sigma)]
=\frac{\vdata_A(\sigma)+\cskip^2(\sigma)\vdata_Y(\sigma)-2\cskip(\sigma)\cdata(\sigma)}
{\cout^2(\sigma)}.
\]
\end{proof}
To summarize, $\cin, \ssin$ were chosen so that the input to the network has mean zero and unit variance. As for $\cbias, \cout$ , for a given choice of $\cskip$, we choose $\cbias, \cout$  so that the target has mean zero and unit variance. It remains to choose $\cskip$, which we do by minimizing the residual variance.
\begin{claim}[$\cskip$]
\label{cl:opt_cskip}
The minimizer over $a\in\mathbb R$ of
\[
R_{\sigma}(a):=\var[A_{\sigma}^i-aY_{\sigma}^i]
=\vdata_A(\sigma)+a^2\vdata_Y(\sigma)-2a\cdata(\sigma)
\]
is
\begin{equation}
\label{eq:cskip}
\cskip(\sigma):=\frac{\cdata(\sigma)}{\vdata_Y(\sigma)}.
\end{equation}
\end{claim}
\begin{proof}
The quadratic $R_{\sigma}(a)$ has derivative $2a\vdata_Y(\sigma)-2\cdata(\sigma)$, which vanishes at \cref{eq:cskip}.
\end{proof}
\begin{claim}[Resulting $\cout$]
\label{cl:opt_cout}
If $\cskip(\sigma)$ is as in \cref{eq:cskip} then
\begin{equation}
\label{eq:cout}
C(\sigma):=\min_a R_{\sigma}(a)
=\vdata_A(\sigma)-\frac{\cdata^2(\sigma)}{\vdata_Y(\sigma)}\qquad\textnormal{and}
\qquad \cout(\sigma)=\sqrt{C(\sigma)}.
\end{equation}
\end{claim}
\begin{proof}
Substitute \cref{eq:cskip} into \cref{eq:out_skip_rel}.
\end{proof}
Finally, we choose the effective squared-error weight $\w^2(\sigma)$ so that the expected loss at initialization $\net_{\theta}=0$ is constant in $\sigma$.
\begin{claim}[$\w$]
\label{cl:w}
If
\begin{equation}
\label{eq:w}
\w^2(\sigma)=\frac{1}{C(\sigma)}=\frac{1}{\cout^2(\sigma)},
\end{equation}
then, for all $\sigma$,
\[
\sum_{a,y}\w^2(\sigma)\,|a-\cskip(\sigma)y-\cbias(\sigma)|^2\,\pi_{\sigma}(a,y)=\dd.
\]
\end{claim}
\begin{proof}
Each residual coordinate has zero mean and variance $C(\sigma)$, so its second moment is $C(\sigma)$. Summing over coordinates and multiplying by $1/C(\sigma)$ gives $\dd$.
\end{proof}
We summarize the EDM scalings in \cref{tab:edm_generic}.
\begin{table}[htbp]
    \centering
    \caption{EDM scalings for target $A_{\sigma}$ and input $Y_{\sigma}$.}
    \label{tab:edm_generic}
    \setlength{\tabcolsep}{12pt}
    \renewcommand{\arraystretch}{1.8}
    \begin{tabular}{l|c}
        \toprule
        $\cin(\sigma)$ & $\frac{1}{\sqrt{\vdata_Y(\sigma)}}$ \\
        $\ssin(\sigma)$ & $-\frac{\mdata_Y(\sigma)}{\sqrt{\vdata_Y(\sigma)}}$ \\
        \midrule
        $\cskip(\sigma)$ & $\frac{\cdata(\sigma)}{\vdata_Y(\sigma)}$ \\
        $\cbias(\sigma)$ & $\mdata_A(\sigma)-\cskip(\sigma)\mdata_Y(\sigma)$ \\
        $C(\sigma)$ & $\vdata_A(\sigma)-\frac{\cdata^2(\sigma)}{\vdata_Y(\sigma)}$ \\
        $\cout(\sigma)$ & $\sqrt{C(\sigma)}$ \\
        $\w^2(\sigma)$ & $\frac{1}{C(\sigma)}$ \\
        \bottomrule
    \end{tabular}
\end{table}

\subsection{EDM scaling for \jud}
In this section we instantiate the EDM scaling for \jud, both for the Binomial-Poisson perturbation and for the Poisson-Poisson perturbation. This requires us to compute the statistics for the observations $\pp_t$ and $\ppp_t$, as well as the targets $\data, \binnrv_t, \possrv_t, \mathrm L_t$. The following claim will be used to this end. 
\begin{claim}[Means and variances for the Binomial and Poisson perturbations]
\label{cl:mean_var_BP}
Let $Z_0$ be a random vector with i.i.d. coordinates with mean $\mdata$ and variance $\vdata$. 
\begin{enumerate}
\item
  Let $Z_{\sigma}=\binnrv_{\sigma}$ with $\binnrv_{\sigma}\sim\bin(Z_0,\Rd_{\sigma})$. Then,
\begin{equation}
\label{eq:EDM_bp_latent_app}
\begin{split}
\EE[Z_{\sigma}]=\Rd_{\sigma}\mdata,\quad \var[Z_{\sigma}] =\mdata \Rd_{\sigma}(1-\Rd_{\sigma})+\Rd_{\sigma}^2 \vdata,\quad  \cov(Z_0,Z_{\sigma})=\Rd_{\sigma}\vdata.
\end{split}
\end{equation}
\item Let $Z_{\sigma}=Z_0\pm \possrv_{\sigma}$ with  $\possrv_{\sigma}\sim \pos(\Ru_{\sigma})$ where  $\possrv_{\sigma}$ is independent of $Z_0$. 
Then,
\begin{equation}
\label{eq:EDM_bp_poisson_app}
\begin{split}
\EE[Z_{\sigma}]=\mdata\pm \Ru_{\sigma},\quad \var[Z_{\sigma}] =\vdata+\Ru_{\sigma},\quad  \cov(Z_0,Z_{\sigma})=\vdata.
\end{split}
\end{equation}
\item Let $Z_{\sigma}=\binnrv_{\sigma}+\possrv_{\sigma}$ with $\binnrv_{\sigma}\sim\bin(Z_0,\Rd_{\sigma})$, $\possrv_{\sigma}\sim \pos(\Ru_{\sigma})$ where  $\possrv_{\sigma}$ is independent of $(Z_0,\binnrv_{\sigma})$. Then,
\begin{equation}
\label{eq:EDM_bp_data_app}
\EE[Z_{\sigma}]=\mdata \,\Rd_{\sigma}+\Ru_{\sigma},\quad \var[Z_{\sigma}] =\mdata \Rd_{\sigma}(1-\Rd_{\sigma})+\Rd_{\sigma}^2\vdata+\Ru_{\sigma},\quad 
\cov(Z_0,Z_{\sigma})=\Rd_{\sigma}\vdata.
\end{equation}
\end{enumerate}
\end{claim}
\begin{proof}
$~$
\begin{enumerate}
\item 
Using the mean and variance  of the binomial distribution,
\begin{equation*}
\EE[Z_{\sigma}] = \EE[\EE[Z_{\sigma}|Z_0]]=\EE[Z_0\,\Rd_{\sigma}]= \Rd_{\sigma}\mdata,
\end{equation*}
\begin{align*}
\var[Z_{\sigma}]&=\EE[\var[Z_{\sigma}|Z_0]]+\var(\EE[Z_{\sigma}|Z_0])=\EE[Z_0\,\Rd_{\sigma}(1-\Rd_{\sigma})]+\var(Z_0\,\Rd_{\sigma})\\
&=\mdata \Rd_{\sigma}(1-\Rd_{\sigma})+\Rd_{\sigma}^2 \vdata,
\end{align*}
 and
\begin{align*}
\cov(Z_0,Z_{\sigma})&=\EE[Z_0Z_{\sigma}]-\EE[Z_0]\EE[Z_{\sigma}]=\EE[Z_0\EE[Z_{\sigma}|Z_0]]-\mdata\cdot \Rd_{\sigma}\mdata\\
&=\Rd_{\sigma}\EE[Z_0^2]-\mdata^2\Rd_{\sigma}=\Rd_{\sigma}(\vdata+\mdata^2)-\mdata^2\Rd_{\sigma}=\Rd_{\sigma}\vdata.
\end{align*}
\item Using the mean and variance  of the Poisson distribution, and  the independence of $\possrv_{\sigma}$ from $Z_0$,
\begin{align*}
\EE[Z_{\sigma}] = \EE[\EE[Z_{\sigma}|Z_0]]=\EE[\EE[Z_0\pm \possrv_{\sigma}|Z_0]]=\mdata\pm \Ru_{\sigma},
\end{align*}
\begin{align*}
\var[Z_{\sigma}]&=\EE[\var[Z_{\sigma}|Z_0]]+\var(\EE[Z_{\sigma}|Z_0])=\EE[\var[Z_0\pm \possrv_{\sigma}|Z_0]]+\var(\EE[Z_0\pm \possrv_{\sigma}|Z_0])\\
&=\EE[\var[\pm \possrv_{\sigma}]]+\var(Z_0\pm \Ru_{\sigma})=\Ru_{\sigma}+\vdata,
\end{align*}
and 
\begin{align*}
\cov(Z_0,Z_{\sigma})&=\EE[Z_0Z_{\sigma}]-\EE[Z_0]\EE[Z_{\sigma}]=\EE[Z_0(Z_0\pm \possrv_{\sigma})]-\EE[Z_0]\EE[Z_0\pm \possrv_{\sigma}]\\
&=\EE[Z_0^2]\pm\mdata\Ru_{\sigma}-\mdata(\mdata\pm \Ru_{\sigma})=(\vdata+\mdata^2)-\mdata^2=\vdata.
\end{align*}
\item Using the means and variances  of the Binomial and Poisson distributions, and the independence of $\possrv_{\sigma}$ from $(Z_0,\binnrv_{\sigma})$,
\begin{align*}
\EE[Z_{\sigma}] &= \EE[\EE[Z_{\sigma}|Z_0]]=\EE[\EE[\binnrv_{\sigma}+\possrv_{\sigma}|Z_0]]=\EE[\EE[\binnrv_{\sigma}|Z_0]+\possrv_{\sigma}]\\
&=\EE[Z_0\,\Rd_{\sigma}]+\Ru_{\sigma}=\mdata \,\Rd_{\sigma}+\Ru_{\sigma},
\end{align*}
\begin{align*}
\var[Z_{\sigma}]&=\EE[\var[Z_{\sigma}|Z_0]]+\var(\EE[Z_{\sigma}|Z_0])=\EE[\var[\binnrv_{\sigma}+\possrv_{\sigma}|Z_0]]+\var(\EE[\binnrv_{\sigma}+\possrv_{\sigma}|Z_0])\\
&=\EE[\var[\binnrv_{\sigma}|Z_0]]+\var[\possrv_{\sigma}]+\var(\EE[\binnrv_{\sigma}|Z_0]+\EE[\possrv_{\sigma}])\\
&=\EE[Z_0\,  \Rd_{\sigma}(1-\Rd_{\sigma})]+\Ru_{\sigma}+\var(\EE[\binnrv_{\sigma}|Z_0])=\mdata \Rd_{\sigma}(1-\Rd_{\sigma})+\Ru_{\sigma}+\var(Z_0\, \Rd_{\sigma})\\
&=\mdata \Rd_{\sigma}(1-\Rd_{\sigma})+\Ru_{\sigma}+\Rd_{\sigma}^2\vdata,
\end{align*}
 and
\begin{align*}
\cov(Z_0,Z_{\sigma})&=\EE[Z_0Z_{\sigma}]-\EE[Z_0]\EE[Z_{\sigma}]=\EE[Z_0\EE[\binnrv_{\sigma}+\possrv_{\sigma}|Z_0]]-\mdata \EE[\binnrv_{\sigma}+\possrv_{\sigma}]\\
&=\EE[Z_0(\EE[\binnrv_{\sigma}|Z_0]+\EE[\possrv_{\sigma}])]-\mdata (\mdata\Rd_{\sigma}+\Ru_{\sigma})\\
&=\EE[Z_0(Z_0\Rd_{\sigma}+\Ru_{\sigma})]-\mdata (\mdata\Rd_{\sigma}+\Ru_{\sigma})=\Rd_{\sigma}\EE[Z_0^2]+\mdata \Ru_{\sigma}-\mdata^2\Rd_{\sigma}-\Ru_{\sigma}\mdata\\
&=\Rd_{\sigma}(\vdata+\mdata^2)-\mdata^2\Rd_{\sigma}=\Rd_{\sigma}\vdata.
\end{align*}
\end{enumerate}
\end{proof}

\paragraph{Binomial-Poisson moments.}
Let the clean data coordinates have mean $\mdata$ and variance $\vdata$, and write
\[
\pp_t=\binnrv_t+\possrv_t,\qquad
\binnrv_t\mid\data\sim\bin(\data,\Rd_t),\qquad
\possrv_t\sim\pos(\Ru_t),\qquad
\mathrm L_t:=\data-\binnrv_t,
\]
where $\possrv_t$ is independent of $(\data,\binnrv_t)$. Here and below the schedules are scalar, or the formulas are applied coordinatewise. Define
\[
V_B(t):=\mdata\Rd_t(1-\Rd_t)+\Rd_t^2\vdata,\qquad
V_L(t):=\mdata\Rd_t(1-\Rd_t)+(1-\Rd_t)^2\vdata.
\]
We apply the generic EDM formulas with $\sigma=t$ and $Y_{\sigma}=\pp_t$. By \cref{cl:mean_var_BP}(3), the observation moments are
\begin{equation}
\label{eq:edm_bp_input_moments}
\mdata_Y(t)=\Rd_t\mdata+\Ru_t,\qquad
\vdata_Y(t)=V_B(t)+\Ru_t.
\end{equation}
The target-specific moments are listed in \cref{tab:edm_denoisers},
and follow from \cref{cl:mean_var_BP}, together with the identities
\[
\cov(\mathrm L_t^i,\pp_t^i)
=\cov(\data^i-\binnrv_t^i,\binnrv_t^i+\possrv_t^i)
=\Rd_t\vdata-V_B(t)
=\Rd_t(1-\Rd_t)(\vdata-\mdata),
\]
and
\[
\EE[\possrv_t\mid\pp_t]=\pp_t-\EE[\binnrv_t\mid\pp_t],\qquad
\EE[\mathrm L_t\mid\pp_t]=\EE[\data\mid\pp_t]-\EE[\binnrv_t\mid\pp_t].
\]
\begin{table}[htbp]
    \centering
    \caption{Moments of the regression targets for the Binomial-Poisson perturbation $\pp_t=\binnrv_t+\possrv_t$.}
    \label{tab:edm_denoisers}
    \renewcommand{\arraystretch}{1.5}
    \begin{tabular}{lccc}
        \toprule
        Target $A_t$ & $\mdata_A(t)$ & $\vdata_A(t)$ & $\cdata(t)=\cov(A_t^i,Y_t^i)$ \\
        \midrule
        $\binnrv_t$  & $\Rd_t\mdata$ & $V_B(t)$ & $V_B(t)$ \\
        $\data$  & $\mdata$ & $\vdata$ & $\Rd_t\vdata$ \\
        $\possrv_t$  & $\Ru_t$ & $\Ru_t$ & $\Ru_t$ \\
        $\mathrm L_t$  & $(1-\Rd_t)\mdata$ & $V_L(t)$ & $\Rd_t(1-\Rd_t)(\vdata-\mdata)$ \\
        \bottomrule
    \end{tabular}
\end{table}
\paragraph{Poisson-Poisson moments.}
Let the clean data coordinates have mean $\mdata$ and variance $\vdata$, and write
\[
\ppp_t=\data-\drv_t+\urv_t,\qquad
\drv_t\sim\pos(\Rd_t),\qquad
\urv_t\sim\pos(\Ru_t),
\]
with $\data,\drv_t,\urv_t$ independent. By \cref{cl:mean_var_BP} the observation moments are
\begin{equation}
\label{eq:edm_pp_input_moments}
\mdata_Y(t)=\mdata-\Rd_t+\Ru_t,\qquad
\vdata_Y(t)=\vdata+\Rd_t+\Ru_t.
\end{equation}
The moments of the targets follow from \cref{cl:mean_var_BP} and can be found in \cref{tab:edm_pp_moments}.

\begin{table}[htbp]
    \centering
    \caption{Moments of the regression targets for the Poisson-Poisson perturbation $\ppp_t=\data-\drv_t+\urv_t$.}
    \label{tab:edm_pp_moments}
    \renewcommand{\arraystretch}{1.5}
    \begin{tabular}{lccc}
        \toprule
        Target $A_t$ & $\mdata_A(t)$ & $\vdata_A(t)$ & $\cdata(t)=\cov(A_t^i,Y_t^i)$ \\
        \midrule
        $\urv_t$  & $\Ru_t$ & $\Ru_t$ & $\Ru_t$ \\
        $\drv_t$  & $\Rd_t$ & $\Rd_t$ & $-\Rd_t$ \\
        \bottomrule
    \end{tabular}
\end{table}

\section{Synthetic implementation details} \label{app:synthetic}

\subsection{Data}
\label{app:subsec_data}
We draw $N=20{,}000$ i.i.d.\ samples %
from each of the following one-dimensional target distributions for the synthetic datasets with support $S$: 
\begin{enumerate}
    \item Poisson ($\lambda=5$, $S=40$): $p(x)\propto \lambda^x e^{-\lambda}/x!$
    \item Poisson mixture ($S=140$): $p(x)\propto \sum_i w_i\lambda_i^x e^{-\lambda_i}/x!$, $w=(0.1,0.9)$, $\lambda=(1,100)$
    \item Zero-inflated Poisson ($S=50$): $p(0)=\pi_0+(1-\pi_0)e^{-\lambda}$, $p(x)=(1-\pi_0)\lambda^xe^{-\lambda}/x!$ for $x>0$; $\pi_0=0.7,\lambda=5$
    \item Negative binomial mixture ($S=150$): $p(x)\propto\sum_i w_i\binom{x+r_i-1}{x}p_i^{r_i}(1-p_i)^x$, $w=(0.8,0.2)$, $r=(1,10)$, $p=(0.9,0.1)$
    \item Beta–negative-binomial ($S=100$): $p\sim\mathrm{Beta}(1.5,1.5)$, $X|p\sim\mathrm{NegBinom}(r{=}5,p)$, marginalized by quadrature
    \item Zipf ($S=50$): $p(x)\propto x^{-\alpha}$, $\alpha=1.7$ %
    \item Yule–Simon ($S=50$): $p(x)\propto \rho B(x,\rho+1)$, $\rho=2$ %
    \item Correlated Beta-Binomial (dim 1, $S=65$): $p\sim\mathrm{Beta}(2,5)$, $X|p\sim\mathrm{Binomial}(64,p)$, giving the Beta-Binomial marginal $\binom{n}{x}B(x+\alpha,n-x+\beta)/B(\alpha,\beta)$
\end{enumerate}

\subsection{Training} 
\label{app:train_subsec_syn}
\paragraph{Model.} We evaluate all models using a common backbone model based on a residual MLP. A scalar input $x\in\mathbb{R}$ is first rescaled using EDM preconditioning by $z = c_{\mathrm{in}}(t)x + s_{\mathrm{in}}(t)$ and linearly projected to a hidden width $d_{\mathrm{model}}=128$. The diffusion time $t\in[T_0,T]$ is embedded via a sinusoidal timestep embedding with base frequency $10^4$, and time-scale factor $100$ to an embedding dimension $d_{\mathrm{temb}}=128$. This is passed through a two-layer MLP (ReLU, hidden width 256) to produce a $4d_{\mathrm{temb}}$-dimensional conditioning vector. The main layers consist of $3$ residual blocks %
with $256$ hidden units followed by FiLM modulation. %
A method-specific head is applied to the final hidden state $h\in\mathbb{R}^{128}$: a regression head for denoisers or a soft-plus for logit predictions. For the models that require two latent denoisers or probabilities to predict birth and death rates, we use a model with two heads that share the same hidden structure.

\paragraph{Training.} We use the Adam optimizer with the learning rate $2\times10^{-4}$, no weight decay, batch size $256$, gradient-norm clipped to $1.0$, EMA decay $0.999$ and train for $300$ epochs. We compute the mean-squared error loss for the Binomial-Poisson model between the model's prediction and the appropriate conditional-mean target, averaged over a batch of times $t\sim\mathrm{Unif}(T_0,T)$ and one forward-noised sample per example. For the count-fm and the denoising score-entropy loss, the loss is based on the Kullback-Leibler divergence. 

\subsection{Sampling} 
\label{app:subsetc_sample_syn}
To evaluate the model, we draw $1000$ samples per method using $2000$-step reverse-time Euler discretization between $T = 1$ and $T_0 = 10^{-4}$. We observe that similar results were obtained with tau-leaping. Each process is initialized from its respective reference distribution that is described in Table~\ref{tab:reference_distributions}. The accuracy is reported as the Wasserstein-1 distance to a fresh batch of $1000$ held-out samples from the true distribution. We averaged the result across $5$ trials: we report the mean and standard deviation of the results.

\begin{table}[t]
\centering
\caption{Reference distributions used by different methods.}
\label{tab:reference_distributions}
\begin{tabular}{l l}
\hline
\textbf{Method} & \textbf{Reference distribution} \\
\hline
\Def{Binomial-Poisson} & $\pos(20)$ \\
\Def{Binomial}~\citep{shenfeld2026binomial}
    & $\delta_0$ \\
\Def{Poisson} & $\pos(196)$ \\
\textbf{Count Bridge}~\citep{fishman2026count}
    & $\ske(20,20)$ \\
\textbf{Count Flow}~\citep{wei2026flow}
    & $\operatorname{Unif}\{0,\ldots,149\}$ \\
\textbf{SEDD}~\citep{lou2024discrete}
    & $\operatorname{Unif}\{0,\ldots,149\}$ \\
\hline
\end{tabular}
\end{table}

\subsection{Details on \cref{fig:toy_paths}} \label{app:details_Fig1}
We use the Negative Binomial mixture (\cref{app:subsec_data}) as the same target distribution for all methods. The network we use across models is the same as in \cref{app:train_subsec_syn}. The Binomial-Poisson model is trained using \cref{alg:bin_poss_train} and the Poisson model trains the data denoiser (\cref{app:only_up_down}). The rest of the methods are trained per the published implementations. Sampling is done as in \cref{app:subsetc_sample_syn}. The reference distributions are described in \cref{tab:reference_distributions}.

\section{CIFAR-10 implementation details}
\label{app:cifar10_details}

\subsection{Training}
\label{app:cifar10_training}

\paragraph{Architecture and optimization.}
\jud{} and our EDM reproduction use DDPM++ with base width $128$, channel multipliers $(2,2,2)$, and dropout $0.13$. Both use Adam with batch size $512$ and learning rate $0.001$, linearly warmed up over $10$ million images and then held constant. The training budget is $200$ million images. EMA uses a target half-life of $500{,}000$ images with ramp-up.

\paragraph{Data and corruption.}
We use the 50,000 CIFAR-10 training images without class labels. Image augmentation is applied before clipping and rounding the pixel values to integers in $[0,255]$. For each image, we draw one perturbation time $t\in[\varepsilon,1-\varepsilon]$, with $\varepsilon=10^{-5}$, and independently sample the coordinatewise perturbations
\[
\lrv_t^i\mid\data\sim\bin(\data^i,t),\quad
\possrv_t^i\sim\pos(\mdata t),\quad
\pp_t:=\data-\lrv_t+\possrv_t.
\]
The time distribution is specified below. The preconditioning uses the fixed scalar data moments $\mdata=120.7$ and $\vdata=4115.2$.

\paragraph{Denoisers and objective.}
The shared network has three input channels and six output channels, split into the  predictions $(\denbin,\denpos)$. We use the centered affine preconditioning in \cref{tab:edm_generic}, with the target moments in \cref{tab:edm_denoisers}, and provide $-\log(1-t+\varepsilon)$ as the network's time input. Writing $C_N(t)$ and $C_L(t)$ for the two residual variances in \cref{eq:cout}, the per-image loss is
\begin{equation}
\label{eq:cifar10_training_loss}
\ell_t=\frac{1}{2\dd}\left(
\frac{\|\denpos(\pp_t,t)-\possrv_t\|_2^2}{C_N(t)+10^{-8}}
+\frac{\|\denbin(\pp_t,t)-\mathrm L_t\|_2^2}{C_L(t)+10^{-8}}
\right),\qquad \dd=3\cdot32^2.
\end{equation}

\paragraph{Training-time distribution.}
As in EDM, we allocate more training examples to noise levels at which the neural denoiser captures structure beyond the affine predictor. We first train a pilot for 5 million images with $t$ uniform on $[\varepsilon,1-\varepsilon]$. We then evaluate its EMA denoisers on 4,096  training images, at 81 times $t_0<\cdots<t_{80}$ uniformly spaced in $\log(t/(1-t))$ between the two endpoints.

The affine baseline is obtained by setting the learned network residual to zero in \cref{eq:edm_scale_app}. Let $\widehat\ell_{\mathrm{aff}}(t_j)$ and $\widehat\ell_{\mathrm{pilot}}(t_j)$ denote its empirical loss and that of the pilot, respectively, both evaluated using \cref{eq:cifar10_training_loss}. Define the positive improvement
\[
g_j=\max\!\left\{\widehat\ell_{\mathrm{aff}}(t_j)-\widehat\ell_{\mathrm{pilot}}(t_j),\,0\right\},\qquad
g(t)=\frac{g_j+g_{j+1}}{2}\quad\text{for }t\in[t_j,t_{j+1}).
\]
For the remaining training, we sample times from a fixed mixture: $10\%$ from a uniform proposal on $[\varepsilon,1-\varepsilon]$ and $90\%$ from a targeted proposal proportional to $g(t)$. The targeted proposal favors noise levels where the pilot improves more over the affine baseline.

\subsection{Sampling}
\label{app:cifar10_sampling}

The CIFAR-10 model uses $\TT=1$ and the scalar schedules, for each $i\in [\dd=3\cdot32^2]$
\[
\Rdi_t^i=t,\qquad \Ru_t^i=\mdata t,\qquad \mdata=120.7.
\]
Sampling starts at $t=1$ (pure noise) and decreases to $t=\varepsilon$, with $\varepsilon=10^{-5}$.

At state $\x$ and time $t$, a single network evaluation produces the preconditioned denoiser estimates $\denbin(\x,t)\approx\EE[\mathrm L_t\mid\pp_t=\x]$ and $\denpos(\x,t)\approx\EE[\possrv_t\mid\pp_t=\x]$. The network's time input is transformed as $-\log(1-t+\varepsilon)$, an adaptation of EDM's time embedding. 

Sampling uses either tau-leaping or Euler as in \Cref{alg:tau} and \Cref{alg:euler}, respectively.

\subsection{Sampling ablations}
\label{app:cifar10_sampling_ablations}

\paragraph{Marginal-preserving rates.}
We augment tau-leaping with the forward-plus-reverse generator correction of \citet{DBLP:journals/corr/abs-2205-14987}, with unit strength. These additional rates preserve the instantaneous marginal distribution when the conditional means are exact. In \cref{alg:tau}, replace the rates at $0<t<1$ by
\begin{equation}
\label{eq:cifar10_campbell_rates}
\lambda_{+,\mathrm C}^i=\mdata+\frac{2\,\denbin^i(\x,t)}{t},\qquad
\lambda_{-,\mathrm C}^i=\frac{\x^i-\denpos^i(\x,t)}{1-t}+\frac{2\,\denpos^i(\x,t)}{t}.
\end{equation}
At $t=1$, the implementation replaces $\frac{\x^i-\denpos^i(\x,t)}{1-t}$ by $\x^i$. The Poisson draws and state update are unchanged. The modified rates reuse the same denoiser predictions, so there is no additional network evaluation.

\paragraph{Results.}
\Cref{tab:cifar10_sampling_ablations} varies the sampler and number of steps while keeping the trained model fixed. At 256 steps, tau-leaping obtains a substantially lower FID than Euler. At 1024 and 2048 steps, tau-leaping again obtains a slightly lower FID. Adding marginal-preserving rates worsens FID at every evaluated step count.

\begin{table}[htbp]
\centering
\caption{CIFAR-10 FID at different numbers of sampling steps (NFE).}
\label{tab:cifar10_sampling_ablations}
\small
\renewcommand{\arraystretch}{1.15}
\begin{tabular}{lrrrr}
\toprule
Sampler & 256 & 512 & 1024 & 2048 \\
\midrule
Euler & 7.07 & 2.22 & 2.15 & 2.09 \\
tau-leaping & 3.27 & 2.29 & 2.11 & \textbf{2.07} \\
\midrule
tau-leaping + marginal-preserving rates & 5.46 & 2.95 & 2.34 & 2.20 \\
\bottomrule
\end{tabular}
\end{table}

\section{Sudoku implementation details}
\label{app:sudoku_details}

\subsection{Training configuration}
\label{app:sudoku_training}
All four variants use an eight-block DiT with width $512$, MLP width $2048$, eight attention heads, and dropout $0.1$. We use AdamW, batch size $256$, peak learning rate $3\times10^{-4}$ with warmup, gradient-norm clipping at $1$, and EMA. MSE and CE share the following settings within each digit encoding.
\begin{center}
\begin{tabular}{lrll}
\toprule
Setting & Updates & Learning-rate schedule & Weight decay \\
\midrule
Original digits & $10{,}000$ & Warmup, then cosine decay & $0.05$ \\
Shuffled digits & $20{,}000$ & Warmup, then constant & $0$ \\
\bottomrule
\end{tabular}
\end{center}
Training times are sampled uniformly from $[0.001,0.999]$. Digit shuffling draws a fresh independent permutation for each board presentation, applied jointly to the puzzle and solution.

\subsection{Posterior parametrization}
\label{app:sudoku_posterior}

Let $\mathcal C$ denote the puzzle context. For each cell $i$, the cross-entropy network outputs nine logits whose softmax approximates the clean-digit posterior,
\begin{equation}
\label{eq:sudoku_clean_posterior}
p_\theta^i(k\mid\x,t,\mathcal C)
\approx \PP(\data^i=k\mid\pp_t=\x,\mathcal C),
\qquad k\in\{1,\ldots,9\}.
\end{equation}
The conditioning includes the entire perturbed solution and puzzle. The EDM-MSE output scalings are not applied to the logits.

We use the Binomial-Poisson perturbation, $\pp_t^i=\data^i-\lrv_t^i+\possrv_t^i$, with $\lrv_t^i
|\data^i\sim\bin(\data^i,\Rdi_t^i)$ and independent $\possrv_t^i\sim\pos(\Ru_t^i)$, with $\Rdi_t^i=t$ and $\Ru_t^i=5t$ for all $i$. 

\paragraph{Posterior averaging.}

As shown in \cref{lem:cor_cond_exp_binpos_posterior}, we have 
\begin{equation}
\label{eq:latent_Den_beta}
\den^i(\x,t)=\sum_{k=1}^9\trans_{\data^i|\pp_t=\x}(k)\beta_t(\x^i,k)
\end{equation}
where
\begin{equation}
\label{eq:def_beta}
\beta_t(\x^i,k):=\sum_{b=0}^{\min(\x^i,k)}b\, \frac{\binn_{k,\Rd_t^i}(b)\poss_{\Ru_t^i}(\x^i-b)}{(\binn_{k,\Rd_t^i}\ast \poss_{\Ru_t^i})(\x^i)}.
\end{equation}
Hence, replacing $\trans_{\data^i|\pp_t=\x}(k)$ by the approximation $p_\theta^i(k\mid\x,t,\mathcal C)$, and using the relations \cref{eq:latents_data_latents}-\cref{eq:latents_data_latents2}, we have
\begin{equation}
\label{eq:sudoku_posterior_nl_means}
\begin{aligned}
\denpos^i(\x,t;\mathcal C)
&\approx\sum_{k=1}^{9}p_\theta^i(k\mid\x,t,\mathcal C)
\bigl[\x^i-\beta_t(\x^i,k)\bigr],\\
\denbin^i(\x,t;\mathcal C)
&\approx\sum_{k=1}^{9}p_\theta^i(k\mid\x,t,\mathcal C)
\bigl[k-\beta_t(\x^i,k)\bigr].
\end{aligned}
\end{equation}

\paragraph{Sampling.}
The reverse-time jump rates evaluated at forward noise level $0<t\le1$ are
\begin{equation}
\label{eq:sudoku_posterior_rates}
\lambda_+^i(\x,t,\mathcal C)=\frac{\denbin^i(\x,t;\mathcal C)}{t},
\qquad
\lambda_-^i(\x,t,\mathcal C)=\frac{\denpos^i(\x,t;\mathcal C)}{t}.
\end{equation}
These rates enter the tau-leaping update in \cref{alg:tau}.

\section{Related works}
\label{sec:related}

\paragraph{Diffusion models: Continuous data}

Diffusion models in discrete time were first introduced in \citet{sohl2015deep} with an evidence lower bound (ELBO) training objective. The method was then further developed in \citet{ho2020denoising, song2021denoising}, where the connection to denoising score matching was made, with additionally better parameterizations and sampling, which ultimately led to a state-of-the-art (at the time) FID score of 3.17 on CIFAR-10. The  progression of diffusion models to continuous-time as reverse-time stochastic differential equations was developed in \citet{song2021scorebased}, which also unified diffusion models with the score matching with Langevin dynamics framework \citep{song2019generative}. \citet{song2021scorebased} wrote the training objective as denoising score matching, and lowered the FID score on CIFAR-10 to 2.20. %
In \cite{karras2022elucidating}, the design space for training the denoiser was explored systematically, which further brought down the FID to  1.97 on CIFAR-10. The EDM framework developed in \cite{karras2022elucidating} is regularly used in state-of-the-art results for image generation. 
\paragraph{Discrete diffusion models: Categorical data} Repeating history, the first diffusion models for categorical data were developed in discrete time \citep{hoogeboom2021argmax,austin2021structured}, while the   continuous time formulation was later developed in \cite{DBLP:journals/corr/abs-2205-14987}. As for continuous data, the early versions of discrete diffusion models were trained using ELBO-based objectives, specifically training the posterior distribution (not the denoiser). Score-based objectives were developed in \cite{meng2022concrete, lou2024discrete}, and rate-based training was developed in \cite{holderrieth2025generator}. 

\paragraph{Discrete diffusion models: Ordinal data}
Already at the origin of diffusion models, ordinal data appeared in the form of \emph{binary} data \citep{sohl2015deep}, with later developments in \cite{bach2025sampling, benedict2023recfusion, pham2025discrete}. In some biological applications, e.g., gene expression counts in single-cell RNA sequencing, recent works   \citep{zhang2026lingshu, wang2026scdiva,bhattacharya2026discrete,wang2026x} argued that it is beneficial to develop diffusion models that act directly on the discrete data, but these works neglect the ordinal structure and use categorical discrete diffusion models. 

The theory and practice for the  ordinal data considered in this work, namely (subsets of) $\N^{\dd}$ or $\Z^{\dd}$, has only recently started to evolve. Let us begin with data on $\N^{\dd}$.  One approach, developed in \cite{chen2023learning, montanari2023sampling, DBLP:journals/corr/abs-2505-05082}, learns both a denoiser and a rate via a discrete Tweedie's formula, but using perturbation of the form $\data\mapsto \pp_t\sim \pos(t\,\data)$. This type of perturbation can only recover the clean data in the limit $\TT\to\infty$, in contrast to our approach which recovers $\data$ at a final time. In addition, this perturbation can easily overshoot the effective support of the data. Alternative approaches \citep{santos2023blackout, santos2025discrete, shenfeld2026binomial,cheng2026variable}
 use only Binomial noise, which suffers from being a uni-directional perturbation. Another uni-directional perturbation was developed in \citet{campbell2023transdimensional}. In \citet{winkler2024bridging, wei2026flow}, an up-down perturbation based on Binomial noise is developed. However, this formulation does not allow for a direct training of a denoiser.  In \citet{fishman2026count}, the Poisson-Poisson perturbation \eqref{eq:perturb_pp_intro} is used, but the denoiser formulation of \cref{thm:main_p_intro} is lacking, which leads \citet{fishman2026count} to a \emph{distributional} training objective. Similarly, the concurrent work~\citet{wei2026stochastic} considers a Binomial-Poisson perturbation~\eqref{eq:perturb_intro_thm_main}, but is also lacking a denoiser representation and trains using the generator matching framework~\citep{holderrieth2025generator}. To summarize, our work is the first to train a diffusion model for discrete ordinal data which allows for up-down perturbations, while training only denoisers. 
\end{document}

%% file: math_commands.tex
\usepackage{amsmath,amsfonts,bm}

\def\eqref#1{equation~\ref{#1}}

\def\1{\bm{1}}

\def\rd{{\textnormal{d}}}

\def\ru{{\textnormal{u}}}

\DeclareMathAlphabet{\mathsfit}{\encodingdefault}{\sfdefault}{m}{sl}
\SetMathAlphabet{\mathsfit}{bold}{\encodingdefault}{\sfdefault}{bx}{n}

\newcommand{\R}{\mathbb{R}}